%% file: neurips_2026.tex
\documentclass{article}

\usepackage[preprint]{neurips_2026}

\usepackage[utf8]{inputenc} 
\usepackage[T1]{fontenc}    
\usepackage{hyperref}       
\usepackage{url}            
\usepackage{booktabs}       
\usepackage{amsfonts}       
\usepackage{nicefrac}       
\usepackage{microtype}      
\usepackage{xcolor}         
\usepackage{subcaption}
\usepackage{graphicx}

\usepackage{amsmath}
\usepackage{amssymb}
\usepackage{mathtools}
\usepackage{multirow}
\usepackage{amsthm}
\usepackage{algorithmic}
\usepackage{algorithm}

\usepackage[capitalize,noabbrev]{cleveref}

\usepackage{enumitem}
\theoremstyle{plain}

\newtheorem{theorem}{Theorem}[section]
\newtheorem{proposition}[theorem]{Proposition}
\newtheorem{lemma}[theorem]{Lemma}
\newtheorem{corollary}[theorem]{Corollary}
\theoremstyle{definition}
\newtheorem{definition}[theorem]{Definition}

\theoremstyle{remark}

\newtheorem{example}{Example}[section]
\newcommand{\toolname}{D\textsuperscript{3}PO}

\usepackage{totcount}\regtotcounter{fixmeTotalCounter}
\usepackage[useregional]{datetime2}
\input{z-commands}

\title{Preference Conditioned Multi-Objective Reinforcement Learning: Decomposed, Diversity-Driven Policy Optimization}

\author{%
 Tanmay Ambadkar\thanks{https://ambadkar.com/d3po},\quad Sourav Panda,\quad Shreyash Kale,\quad Jonathan Dodge,\quad Abhinav Verma\\
  Pennsylvania State University\\
  University Park, PA 16802 \\
  \texttt{\{ambadkar, sbp5911, shreyash, dodge, verma\}@psu.edu} \\
}

\begin{document}

\maketitle

\begin{abstract}
Multi-objective reinforcement learning (MORL) seeks to train agents capable of balancing conflicting objectives. While single preference-conditioned policies offer a highly scalable solution, existing approaches remain brittle in practice, frequently failing to recover dense Pareto fronts. We demonstrate that this failure stems from two structural pathologies: destructive advantage cancellation caused by premature Early Scalarization (ES), and representational mode collapse across the preference space. To overcome these bottlenecks, we introduce \toolname{}, a PPO-based framework that fundamentally reorganizes multi-objective optimization. By preserving per-objective learning signals through a decomposed pipeline and integrating preferences only after trust-region stabilization (Late-Stage Weighting), \toolname{} improves credit assignment under conflicting objectives. Concurrently, a scaled diversity regularizer encourages behavioral divergence proportional to preference distance. \toolname{} operates entirely within the efficient linear scalarization regime shared by standard deep MORL baselines. By reducing information loss caused due to linear scalarization rather than relying on expensive non-linear utility functions, it suggests that optimization bottlenecks play a significant role. Across available standard benchmarks, including high-dimensional and many-objective environments, \toolname{} consistently discovers broader, higher-quality Pareto fronts than prior methods, exceeding state-of-the-art hypervolume and expected utility using a single deployable policy.
\end{abstract}

\input{documentBody/1-intro}
\input{documentBody/2-related}

\input{documentBody/3-problem}
\input{documentBody/4-method}
\input{documentBody/5-theory}
\input{documentBody/6-experiments}
\input{documentBody/7-conclusion}

\bibliography{main}
\bibliographystyle{unsrt}

\newpage
\appendix
\onecolumn
\input{documentBody/A-code}
\input{documentBody/B-metrics}
\input{documentBody/C-discrete}

\input{documentBody/D-Ablation}
\input{documentBody/E-gradient_analysis}
\input{documentBody/F-theory}
\input{documentBody/G-regularizer}
\input{documentBody/I-environments}

\input{documentBody/J-expDetails}

\input{documentBody/K-Reproducibility}
\input{documentBody/L-demo}

\end{document}

%% file: z-commands.tex
\newcommand{\algName}[1]{MO-PPO}

\newcommand{\REDACT}[1]{$\Box REDACTED \Box$} 

\newcommand{\redactCollege}[1]{[a U.S. University]}  

\newcounter{boldifyCounter}
\newcounter{fixmeSectionCounter}
\newcounter{fixmeTotalCounter}
\makeatletter
\@addtoreset{fixmeSectionCounter}{section}
\@addtoreset{fixmeSectionCounter}{subsection}
\@addtoreset{boldifyCounter}{section}
\@addtoreset{boldifyCounter}{subsection}
\makeatother

\newcommand{\boldify}[1]{}
\ifdefined\boldifyON
	\renewcommand{\boldify}[1]{
        \par\noindent
		\stepcounter{boldifyCounter}
		\textbf{{\color{green}**}
		~\arabic{section}.\arabic{subsection}.\arabic{boldifyCounter}
		: #1} 
	}
\fi 

\newcommand{\reportOnFIXME}{%
    \newcount\iterCounter
    \iterCounter=1
    \newcount\endCounter
    \endCounter=\totvalue{fixmeTotalCounter}
    \advance \endCounter +1
    There are 
    {\color{red}\total{fixmeTotalCounter}} 
    FIX\_ME\\
    links:
    \loop
        {\#\the\iterCounter}
        \advance \iterCounter +1
    \ifnum \iterCounter < \endCounter
    \repeat
}

\newcommand{\FIXME}[1]{} 
\ifdefined\fixmeON
	\renewcommand{\FIXME}[1]{\par\noindent
		\stepcounter{fixmeSectionCounter}\stepcounter{fixmeTotalCounter}
		{\color{red}\fbox{\color{black}
			\parbox{.965\linewidth}{
				\textbf{%
                {FIXME}	\arabic{section}.\arabic{subsection}.
        		\arabic{fixmeSectionCounter} (\color{red}
        		\#\arabic{fixmeTotalCounter}):} #1}}
        }
	}
\fi 

\newcommand{\FIXED}[1]{}
\ifdefined\fixmeON
	\renewcommand{\FIXED}[1]{\par\noindent%
		{\color{black}\fbox{\color{black}%
			\parbox{.99\columnwidth}{%
				\color{blue}#1}}%
        }
	}
\fi 

\newcommand{\draftStatus}[1]{}
\ifdefined\draftStatusON
	\renewcommand{\draftStatus}[1]{
        \hfill **#1
	}
\fi 

%% file: documentBody/1-intro.tex
\section{Introduction
\draftStatus{Page Budget:1pg (GOOD) Status: 
TSA Redraft, JED has read (just a few requests for citation), AV please read
}}
\label{sec:intro}

Reinforcement learning (RL) successfully trains agents to make sequential decisions by maximizing a \textit{single scalar reward}~\citep{sutton1998introduction}. 
However, real-world applications (such as autonomous driving or logistics) require agents to simultaneously optimize multiple, often conflicting objectives. 
Optimizing a single reward function collapses this richness, frequently leading to suboptimal behaviors and motivating the field of \textit{Multi-Objective Reinforcement Learning (MORL)}.

MORL decomposes objectives into a \textit{vector of reward signals}. 
When objectives conflict, a single global optimum is unattainable. Instead, the goal is to learn a set of Pareto-optimal policies~\citep{morl_decomposition}, allowing users to select desired trade-offs via \textit{preference weight vectors}~\citep{utility_functions_morl, nonlinear_scalarization}. 
This non-uniqueness of optimal solutions introduces severe algorithmic challenges. The agent must respond to a potentially infinite set of preference queries, and conflicting objective gradients can point in opposing directions, thereby \textit{destabilizing policy updates and impairing sample efficiency}~\citep{PSL-MORL}. 
In fact, our empirical diagnostics reveal that per-objective gradients actively oppose each other (exhibiting negative cosine similarity) in 36 to 53 percent of all training steps across continuous control benchmarks.

\input{figure/0-teaser}

Consequently, existing methods face persistent limitations. Single-policy methods often suffer from \textbf{destructive gradient interference}, where naively combining conflicting objectives produces opposing gradients that stall learning~\citep{PSL-MORL}. 
Our analysis shows that \textbf{early scalarization can reduce 60–65\% of the aggregated advantage magnitude per update in our experiments} (Sec ~\ref{sec:scalarization_analysis}). 
These methods also frequently exhibit \textbf{incomplete front coverage} due to representational mode collapse, ignoring shifting preferences to output a safe average behavior. 
Alternatively, multi-policy approaches train discrete collections of policies to cover the front, suffering from \textbf{architectural inefficiency} and massive memory costs that scale poorly with the number of objectives.

We present \toolname{} (Decomposed, Diversity-Driven Policy Optimization), illustrated in Figure~\ref{fig:teaser}, a framework for training a single preference-conditioned policy for multi-objective reinforcement learning. 
\toolname{} achieves state-of-the-art performance on the evaluated continuous control benchmarks, outperforming both single-policy baselines and parameterized multi-policy ensembles while utilizing up to 99 percent fewer parameters.

\toolname{} achieves this by revisiting how  Proximal Policy Optimization (PPO) interacts with multi-objective scalarization. Instead of fighting the optimizer by injecting explicit multi-objective gradient surgeries that inherently violate PPO's trust region, \toolname{} leverages the algorithm's native mechanics. By delaying linear scalarization to the late stage (LSW) and injecting non-homogeneity via per-objective advantage normalization, we structurally rely on PPO's clipping operator to act as a \textit{dynamic piecewise non-linearity}. Our hierarchy lemma ($\text{LSW} \;\succ\; \text{MVS} \;\succeq\; \text{ES},$~\ref{cor:hierarchy}) shows that early scalarization can reduce usable gradient signal under conflicting objectives, whereas LSW preserves more per-objective signal prior to aggregation. When these conflicting signals demand destructively divergent updates, independent clipping truncates them, mitigating gradient cancellation. 

\FIXME{TSA: Need to talk about LS}
We demonstrate that multi-policy ensembles struggle to converge discrete policies across complex trade-off manifolds. \toolname{} establishes a new state-of-the-art in Hypervolume and Expected Utility. By approximating a dense, continuous Pareto manifold, our single-policy architecture demonstrates improved trade-off discovery on the evaluated benchmarks. 

Our core contributions are:

\textbf{1. Deployment-efficient single-policy SOTA.} \toolname{} achieves state-of-the-art performance (up to 9 objectives), demonstrating that a single continuous network outperforms 200-network ensembles under strict sample budgets. By eliminating the need to store and route discrete policies, \toolname{} drastically reduces memory footprint and deployment complexity.
    
\textbf{2. Decomposed optimization via Late-Stage Weighting (LSW).} We show analytically and empirically that early scalarization collapses per-objective advantage signals into a single scalar signal and leads to substantial reduction in aggregated advantage magnitude (up to 65\% in our experiments).
We show that LSW preserves these conflicting signals prior to aggregation and utilizes PPO clipping as a stabilizing non-linearity.
    
\textbf{3. Scaled diversity regularization.} We introduce a preference-proportional diversity regularizer that encourages preference-dependent behavioral separation and empirically reduces collapse, forcing the single policy to stretch across the full continuum of the Pareto front.

%% file: figure/0-teaser.tex
\begin{figure*}
\centering
        
    \includegraphics[width=\textwidth]{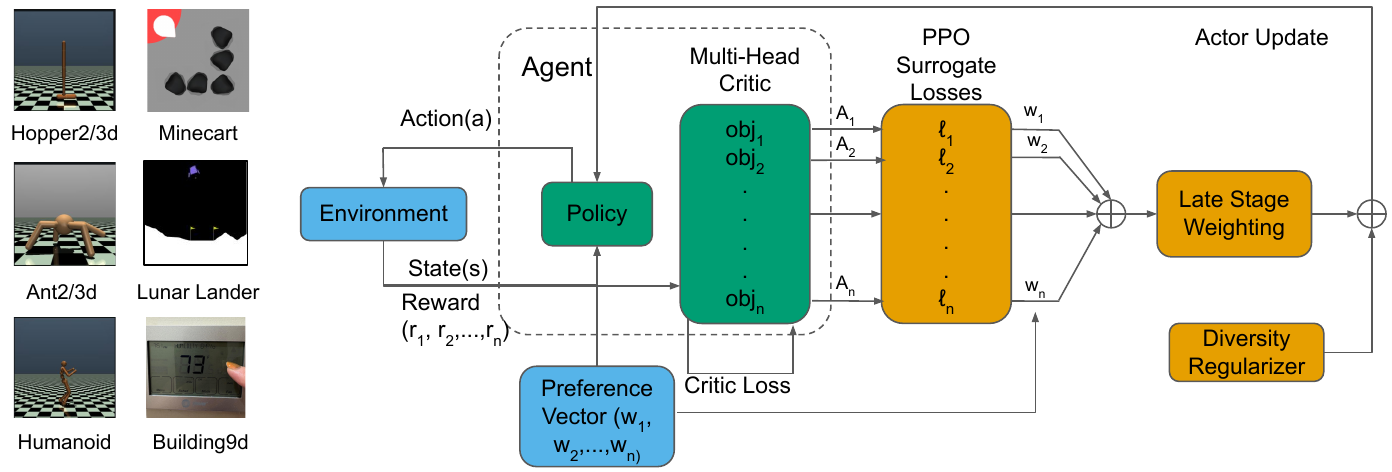}

    \caption{Overview of the $D^{3}PO$ framework. \textbf{(1) Multi-Head Critic:} The critic estimates independent per-objective values $V^{(i)}(s,\omega)$ to compute unweighted advantages $A^{(i)}$. \textbf{(2) PPO Surrogate Losses:} The clipping mechanism is applied to each advantage stream \textit{independently} Eq.~\ref{eq:ppo_loss}, stabilizing the learning signal before scalarization. \textbf{(3) Late-Stage Weighting:} Preference weights $\omega$ are applied only to the stabilized surrogate losses Eq.~\ref{eq:final_loss}, reducing cancellation of gradients prior to optimization. \textbf{(4) Diversity Regularizer:} A diversity loss Eq.~\ref{eq:diversity_loss} is added to force behavioral separation between different preference queries, incurs a penalty under representational collapse, discouraging such solutions during optimization.}

    \label{fig:teaser}
\end{figure*}

%% file: documentBody/2-related.tex
\section{Related Work
\draftStatus{TSA Drafted}
}
\label{sec:related}

Multi-objective reinforcement learning (MORL) has developed along several algorithmic paradigms, each with distinct strengths and limitations.

\textbf{Scalarization.} 
A foundational approach is scalarization, which reduces vector rewards to a scalar for standard RL methods. Linear scalarization (e.g., weighted sums) is computationally efficient but limited to the convex regions of the Pareto front. Nonlinear scalarization functions~\citep{nonlinear_scalarization,utility_functions_morl,nonlinear_preferences} extend expressivity but still collapse objectives into a single training signal, risking loss of information and instability when objectives conflict.

\textbf{Multi-policy methods.}
Other work trains a set of specialized policies for different preferences, then approximates the Pareto front directly~\citep{distributional_pareto_morl,efficient_pareto_front,pa2d_morl, yang2025preference}. Such approaches often rely on constrained optimization or advanced multi-objective optimization techniques to achieve high-quality fronts, but scale poorly with the number of objectives due to the cost of maintaining many policies and are sample-inefficient due to requiring separate environment interactions for each policy.

\textbf{Decomposition Based Approaches.} 
Reward- and value-decomposition methods form an influential class of approaches in multi-objective reinforcement learning. These methods explicitly learn objective-specific value functions or successor features and recombine them, typically through generalized policy improvement (GPI), to derive policies for different scalarizations without retraining \cite{barreto2017sfs,barreto2019transfer}. Variants based on linear scalarization similarly maintain separate per-objective Q-functions and construct policies by applying improvement operators over decomposed value components \cite{vanmoffaert2014pareto}. More recent work has enhanced GPI-based schemes by prioritizing which decomposed components to update in order to improve sample efficiency \cite{gpi-ls}. While such approaches can be effective, they typically rely on linear recombination assumptions and require maintaining multiple value components or policies, and can incur significant storage/compute overhead and may exhibit limited smooth interpolation across the Pareto front.

\textbf{Single universal policies.}
To avoid training multiple policies, recent methods learn a single policy conditioned on a preference vector, enabling adaptation at runtime~\citep{generalized_morl,pareto_conditioned,basaklar2023pdmorl,PSL-MORL,latent_conditioned}. Examples include Pareto-Conditioned Networks (PCN)~\citep{pareto_conditioned}, which reuse past transitions across preferences for sample efficiency; Preference-Driven MORL (PD-MORL)~\citep{basaklar2023pdmorl}, which combines preference conditioning with off-policy engineering such as replay and HER to scale to continuous control; and latent-conditioned policy gradients~\citep{latent_conditioned}, which embed preferences in a latent space. Other PPO-style explorations (e.g., MOPPO~\citep{Terekhov2024InSFA}) study empirical design choices for conditioned PPO variants. These methods demonstrate the practicality of universal preference-conditioned agents but can suffer from gradient interference or representational collapse.

\textbf{The Early Scalarization Trap.} 
We analyze the optimization pipelines of the compared methods at the gradient level (Section~\ref{sec:scalarization_analysis}) and observe a common structural pattern: many approaches apply linear scalarization before gradient-based optimization. This early aggregation can lead to significant cancellation of per-objective advantage signals under conflicting objectives, reducing the effective learning signal available during training.
D3PO addresses this limitation by delaying preference weighting until after per-objective update stabilization (Late-Stage Weighting, LSW), thereby preserving per-objective signals prior to aggregation and mitigating gradient interference.

\textbf{Our contribution.}
D3PO belongs to the fourth family but differs in two key respects: (i) it is an \emph{on-policy} PPO extension with a multi-head critic that preserves raw per-objective signals and applies preferences only after PPO stabilization (Late-Stage Weighting), and (ii) it introduces a \emph{scaled diversity} regularizer that provides a mechanism against mode collapse. This combination of decomposed advantage preservation, principled preference integration, and diversity regularization offers a theoretically enriched alternative to prior preference-conditioned methods, which have primarily emphasized empirical or off-policy approaches.

%% file: documentBody/3-problem.tex
\section{Preliminaries
\draftStatus{TSA Drafted}
}
\label{sec:problem}

We model decision-making problems with multiple objectives using a \emph{Multi-Objective Markov Decision Process} (MOMDP), formalized as the tuple:
\(
\mathcal{M} = \langle \mathcal{S}, \mathcal{A}, P, R_{1:d}, \Omega, \gamma \rangle,
\)
where $\mathcal{S}$ is the state space, $\mathcal{A}$ is the action space, $P(s' \mid s, a)$ is the transition probability function, $R_i(s, a)$ for $i = 1, \dots, d$ are $d$ objective-specific reward functions, $\Omega := \{ \omega \in \mathbb{R}_{\ge 0}^d | \sum_{i=1}^d \omega_i = 1 \}$ denotes the space of preference weights, and $\gamma \in [0,1)$ is the discount factor.

At each timestep $t$, the agent observes state $s_t$, chooses an action $a_t$, and receives a reward vector $r_t = [R_1(s_t, a_t), \dots, R_d(s_t, a_t)]^\top \in \mathbb{R}^d$. Given a preference vector $\omega \in \Omega$, the overall goal is to find a policy $\pi_w$ that maximizes the expected scalarized return:
\(
\mathbb{E}_{\pi}\left[\omega^\top\sum_{t=0}^{\infty} \gamma^t \cdot  r_t\right].
\)
The unweighted vector return corresponding to a policy $\pi$ is given by: 
\(
G^\pi := \mathbb{E}_{\pi}\left[\sum_{t=0}^{\infty} \gamma^t  r_t\right].
\)


\textbf{Pareto Optimality.} Since no single policy can be optimal for all preferences simultaneously, the goal of MORL is to approximate the \emph{Pareto front}, which is a set of non-dominated policies.

\begin{definition}[Pareto Dominance]
Let $u, v \in \mathbb{R}^d$ be two cumulative return vectors. Then $u$ \emph{dominates} $v$ (denoted $u \succ v$) if $u_i \geq v_i$ for all $i$, and there exists at least one objective $j$ such that $u_j > v_j$.
\end{definition}

\begin{definition}[Pareto-Optimal Policy]
A policy $\pi$ with a return vector $G^\pi \in \mathbb{R}^d$ is \emph{Pareto-optimal} if there is no other policy $\pi'$ such that $G^{\pi'}$ dominates $G^\pi$. 
\end{definition}


To evaluate MORL algorithms, we use key metrics that quantify both the quality and diversity of the learned Pareto front. \textbf{Hypervolume (HV)} measures the volume of the objective space dominated by the discovered front, encouraging both Pareto-dominance and spread. \textbf{Sparsity (SP)} measures the evenness of the discovered solutions along the front, with lower values indicating better coverage. \textbf{Expected Utility (EU)} measures the average performance across a distribution of sampled preference weights.
Together, these metrics assess both the fidelity (HV, EU) and diversity (SP) of the solutions.

%% file: documentBody/4-method.tex
\section{Method
\draftStatus{Page Budget: 2.5pg
Status: TSA drafted}
}
\label{sec:method}

We propose \textbf{Decomposed, Diversity-Driven Policy Optimization (\toolname{})}, an extension of the standard PPO framework for learning a single preference-conditioned policy that operates across a continuous spectrum of user-specified trade-offs.
While prior work on universal policies has demonstrated the feasibility of conditioning on preferences, many approaches rely on early scalarization of multi-objective signals, which can reduce the effective learning signal and introduce gradient interference under conflicting objectives.

\toolname{} addresses this limitation through a decomposed optimization framework that maintains per-objective reward and advantage signals throughout training. Rather than aggregating objectives upfront, the method processes each objective independently before combining them, allowing the optimizer to better utilize conflicting signals. In addition, \toolname{} introduces a diversity-driven regularization term that encourages preference-dependent behavioral differentiation within a single policy. This enables the learned policy to produce distinct behaviors across the preference simplex $\omega \in \mathbb{R}^d \text{ s.t. } \sum \omega = 1, ;\omega \geq 0$.

As illustrated in Figure~\ref{fig:teaser}, the $D^3PO$ framework follows a decomposed optimization pipeline. A \textbf{Multi-Head Critic} estimates objective-specific value functions, which are used to compute per-objective Generalized Advantage Estimates (GAE). These advantage signals are then passed through \textbf{Per-Objective PPO Surrogates}, ensuring that each objective is stabilized before aggregation. The resulting updates are combined via \textbf{Late-Stage Weighting} using the preference vector $\omega$. Finally, a \textbf{Diversity Regularizer} is applied to the actor objective to encourage distinct behaviors across nearby preferences, mitigating representational collapse.

\subsection{Methodological Contributions}

The key contributions of \toolname{} lie in two design modifications that adapt PPO to the multi-objective setting. A detailed description of the full algorithm is provided in Algorithm~\ref{alg:mo-ppo-single} (Appendix~\ref{appendix:code}), along with supporting lemmas and propositions.

\textbf{Decomposed Policy Optimization via Late-Stage Weighting:}
\toolname{} computes the PPO clipped surrogate objective independently for each of the $d$ objective-specific advantages. The resulting per-objective losses are then combined using the preference weights $\omega$ to form the final policy update. By applying weighting only after the clipping operation, this \emph{Late-Stage Weighting (LSW)} ensures that PPO’s trust-region mechanism operates directly on per-objective signals before aggregation. As shown in Proposition~\ref{cor:hierarchy}, this ordering preserves more informative advantage signals under conflicting objectives and reduces cancellation effects associated with Early Scalarization (ES) and Mid-stage Vectorial Scalarization (MVS). This process is illustrated in Figure~\ref{fig:teaser}, where per-objective PPO surrogate losses are combined through preference weighting to construct the final objective.

\textbf{Scaled Diversity Regularization:}
To mitigate representational collapse in the preference-conditioned policy, \toolname{} introduces a diversity regularization term in the actor objective. This term encourages the KL divergence between action distributions to scale with the distance between their corresponding preference vectors. As discussed in Proposition~\ref{lem:diversity_theory}, this regularizer promotes preference-dependent behavioral differentiation by discouraging identical policies for distinct preferences. In practice, it helps the learned policy maintain diverse behaviors across the preference simplex. This component is depicted in Figure~\ref{fig:teaser} as the diversity regularizer added to the LSW-based objective.

\subsection{Per-Objective Advantage and Value Estimation}

Following trajectory collection, we compute the Generalized Advantage Estimate (GAE) independently for each of the $d$ objectives, resulting in a vector-valued advantage $\mathbf{A}t \in \mathbb{R}^d$. The critic network, $V\phi(s, \omega)$, approximates the corresponding per-objective state-value functions and is central to this computation.

The critic adopts a multi-head architecture (Figure~\ref{fig:teaser}, green), where a shared feature encoder processes the state $s$ and preference vector $\omega$, followed by $d$ separate output heads. Each head $V_\phi^{(i)}$ predicts the \textbf{unweighted value} for objective $i$. The critic is trained by minimizing the mean squared error between predicted values and empirical unweighted returns $G_t^{(i)}$ across all objectives.

\textbf{Conditioning on Preferences.}
Although the critic predicts unweighted returns, it is conditioned on the preference vector $\omega$. This is necessary because the policy $\pi(\cdot \mid s, \omega)$ depends on $\omega$, and therefore the state distribution and expected future returns are also preference-dependent. The critic thus estimates $V_{\pi_\omega}^{(i)}(s)$, the expected return for objective $i$ under the preference-conditioned policy, making conditioning on $\omega$ essential for accurate value estimation.

\textbf{Normalization.}
We apply standard reward and advantage normalization techniques to stabilize training. These transformations rescale per-objective signals before they are passed to the policy update, helping balance gradients across objectives with different magnitudes. We further analyze the role of normalization within \toolname{} in comparison to existing approaches, highlighting its interaction with the decomposed optimization pipeline. 

\subsection{Policy Optimization with Decomposed Gradients and Diversity Regularization}

We update the actor network, $\pi_\theta(a \mid s, \omega)$ (Figure~\ref{fig:teaser}, green), over $K$ epochs for each batch. The policy optimization objective combines a decomposed PPO surrogate with a diversity-promoting regularizer to improve generalization across the preference space.

\textbf{Per-Objective Policy Loss.} 
We compute the PPO clipped surrogate objective independently for each of the $d$ advantage estimates (Figure~\ref{fig:teaser}, PPO Surrogate Losses), allowing each objective to be processed before preference aggregation:
\begin{equation}\label{eq:ppo_loss}
\mathcal{L}_{\text{clip}}^{(i)}(\theta) = \mathbb{E}_t\left[
\min\left(\rho_t(\theta) A_t^{(i)}, \;
\text{clip}(\rho_t(\theta), 1 - \epsilon, 1 + \epsilon) A_t^{(i)}\right)
\right]
\end{equation}
where the probability ratio is 
\(
\rho_t(\theta) = \frac{\pi_\theta(a_t \mid s_t, \omega)}{\pi_{\theta_{\text{old}}}(a_t \mid s_t, \omega)}.
\)
This formulation ensures that PPO’s clipping mechanism operates on per-objective advantage signals prior to aggregation, reducing cancellation effects that arise under early scalarization.

\textbf{Role of Clipping.} 
The clipping operation introduces a non-linear transformation of the advantage signal, limiting the influence of large policy updates. Empirically, we observe that a substantial fraction of updates are affected by clipping (Section~\ref{sec:experiments}), indicating that it plays a key role in shaping the optimization dynamics under decomposed objectives.

\textbf{Diversity-Promoting Regularization.}
Preference-conditioned policies may map different preference vectors $\omega$ to similar behaviors. To encourage preference-dependent differentiation, we introduce a diversity regularizer. For each sampled preference $\omega$, we generate a nearby preference $\omega'$ by adding small Gaussian noise and projecting back onto the simplex (Figure~\ref{fig:teaser}, Diversity Regularizer).

We then penalize deviations between the divergence of action distributions and the distance between preferences:
\begin{equation}
\begin{split}
\mathcal{L}_{\text{diversity}}(\theta) = \mathbb{E}_t \Big[
\big( D_{\mathrm{KL}}(\pi_\theta(\cdot \mid s_t, \omega) \,\|\, \pi_\theta(\cdot \mid s_t, \omega')) 
- \alpha \|\omega - \omega'\|_1 \big)^2
\Big]
\end{split}
\label{eq:diversity_loss}
\end{equation}
As discussed in Proposition~\ref{lem:diversity_theory}, this regularizer encourages alignment between preference distance and behavioral divergence, mitigating representational collapse.

\textbf{Final Actor Objective.}
The actor objective combines policy improvement and diversity regularization:
\begin{equation}\label{eq:final_loss}
\mathcal{L}_{\text{actor}}(\theta) = 
- \sum_{i=1}^d \omega_i \mathcal{L}_{\text{clip}}^{(i)}(\theta)
+ \lambda_{\text{div}} \mathcal{L}_{\text{diversity}}(\theta).
\end{equation}

Here, the preference weights $\omega_i$ determine the relative importance of each objective in the final update. By applying these weights after per-objective clipping, the resulting gradient reflects a preference-weighted combination of stabilized objective-specific updates. The coefficient $\lambda_{\text{div}}$ controls the strength of the diversity regularization, balancing policy improvement with preference-dependent behavioral differentiation.

%% file: documentBody/5-theory.tex
\section{Analysis}
\label{sec:analysis}

The performance of \toolname{} arises from a combination of design choices that address two key challenges in preference-conditioned MORL: 
(1) achieving \textbf{stable credit assignment} under conflicting objectives, and 
(2) ensuring that a single policy \textbf{generalizes across the continuous preference manifold} without collapsing. 
Our framework addresses these challenges through three complementary components: decomposed value estimation, late-stage preference integration, and diversity regularization. These design choices are motivated by empirical diagnostics and supported by theoretical analysis (Appendix~\ref{sec:LSW_theory}).

\textbf{Stable Credit Assignment via Decomposition.} 
\toolname{} begins with decomposed optimization at the critic level. A multi-head critic predicts the unweighted expected return $V^{(i)}(s,\omega)$ for each objective $i$, and advantages are computed independently, yielding a vector-valued advantage $\mathbf{A}_t$. This preserves distinct per-objective learning signals prior to aggregation.

When objectives conflict, scalarizing advantages into a single value can reduce the magnitude of the resulting signal due to cancellation effects. As shown in Proposition~\ref{prop:magnitude_bound}, the scalarized advantage $|A_t^\omega|$ is strictly smaller than the sum of individual magnitudes whenever objectives have opposing signs. Empirically (Section~\ref{subsec:advantage_cancellation}), we observe frequent gradient conflict, with negative cosine similarity between objectives occurring in up to 53\% of training steps. Under these conditions, early scalarization can substantially reduce the aggregated advantage magnitude (up to $\sim$60--65\% in our experiments), limiting the effective learning signal available during optimization. This helps explain why methods relying on early scalarization can struggle on complex multi-objective tasks.

\textbf{Preference Integration via Late-Stage Weighting.} 
While decomposition preserves per-objective signals, preference information must still be incorporated to guide learning. A key distinction between methods lies in when scalarization occurs within the optimization pipeline. In Early Scalarization (ES), objectives are combined before the optimization step, which can lead to irreversible cancellation. Mid-stage Vectorial Scalarization (MVS) introduces per-objective preprocessing but still aggregates signals before the update.

\toolname{} instead employs \emph{Late-Stage Weighting (LSW)}, where the PPO clipped surrogate objective is computed independently for each objective and combined only after per-objective updates are stabilized. This ordering allows the optimization process to operate on decomposed signals prior to aggregation, reducing cancellation effects and preserving more informative gradients under conflicting objectives. 

Formally, Proposition~\ref{prop:magnitude_bound} characterizes the signal reduction induced by early scalarization, while Proposition~\ref{prop:es_mvs_equiv} shows that ES and MVS are equivalent in the absence of per-objective preprocessing. Proposition~\ref{prop:mvs_distortion}
 and Corollary~\ref{cor:hierarchy} establish that LSW preserves more signal than ES and MVS under conflicting objectives, providing a principled justification for the ordering:
\[
    \text{LSW} \;\succ\; \text{MVS} \;\succeq\; \text{ES},
\]
in terms of signal preservation prior to aggregation. LSW does not eliminate scalarization, but changes when it occurs within the optimization pipeline, allowing per-objective signals to be processed before aggregation.

\textbf{Mitigating Collapse via Diversity Regularization.} 
Stable credit assignment alone does not guarantee that the learned policy will represent the full range of trade-offs. A common failure mode of preference-conditioned policies is representational collapse, where distinct preferences map to similar behaviors.

To address this, \toolname{} introduces a scaled diversity regularizer that encourages the divergence between action distributions to reflect the distance between their corresponding preference vectors. Proposition~\ref{lem:diversity_theory} shows that this regularizer introduces a penalty when distinct preferences produce similar policies, encouraging preference-dependent behavioral differentiation. While the final solution depends on joint optimization with the PPO objective, this regularization provides a consistent gradient signal that mitigates collapse and improves coverage of the preference space in practice.

\textbf{Convex Coverage and Limits of Linear Scalarization.} 
It is important to situate \toolname{} within the theoretical limits of linear scalarization. Prior work (e.g., CAPQL~\cite{capql}) shows that for discounted MDPs with stationary policies, the achievable objective set is convex, and linear scalarization is sufficient to recover the convex Pareto front. However, in environments with concave Pareto fronts, linear methods (including \toolname{} and standard baselines) are fundamentally limited to the convex coverage set.

Thus, the improvements observed with \toolname{} do not arise from overcoming geometric limitations of scalarization, but from improving optimization within this regime. By preserving per-objective signals and reducing cancellation, \toolname{} more effectively realizes the potential of linear scalarization in practice.

\textbf{Simplicity as Evidence.} The simplicity of LSW is not a limitation but the point. We provide three converging lines of evidence that the dominant failure mode in preference-conditioned MORL is premature scalarization, not insufficient expressivity: (1) empirical diagnostics showing 50–65\% advantage magnitude cancellation at every update step (Figure~\ref{fig:advantage_cancellation_grid}); (2) gradient conflict between objective pairs occurring in 30–55\% of training steps across all benchmarks (Figure~\ref{fig:gradient_prevalence_grid}); and (3) the hierarchy lemma (Corollary~\ref{cor:hierarchy}) proving analytically that ES discards valid gradient signal that LSW preserves whenever the trust region binds under conflicting objectives. That a minimal reordering of existing operations with no new parameters, no quadratic programs, no ensemble overhead yields SOTA by a large margin across all evaluated benchmarks is itself evidence that prior methods were bottlenecked by this pathology rather than by representational capacity.

\textbf{Sample Efficiency and Deployment.} 
Finally, the contrast between \toolname{} and multi-policy methods highlights the benefits of a unified policy representation. Multi-policy approaches train separate networks for different preferences, which can be costly in terms of both data and computation. In contrast, \toolname{} learns a single preference-conditioned policy that shares representations across the preference space, enabling more efficient use of data during training and direct generalization to unseen preferences.

This also improves deployment efficiency. Rather than storing and selecting among multiple policies, \toolname{} maps any preference vector $\omega$ directly to an action through $\pi(a \mid s, \omega)$, enabling continuous and efficient adaptation across the preference space with significantly lower parameter overhead.

%% file: documentBody/6-experiments.tex
\section{Experiments
\draftStatus{Page Budget: 2.5pg
Status: TSA drafted}}
\label{sec:experiments}

\input{table/1-monolithic} 

\input{figure/1-fronts}

We evaluate our proposed method, \textbf{\toolname{}}, against state-of-the-art baselines to answer two key questions:
(1) Does \toolname{} achieve comprehensive Pareto front coverage?
(2) Does it effectively reduce mode collapse and generate diverse solutions?

Our evaluation uses a suite of challenging MORL tasks from the MO-Gymnasium library~\citep{mo-gymnasium}, including five continuous control and two discrete control environments, and additionally the Building-9d environment, introduced in~\cite{cmorl}.
We compare \toolname{} against six strong baselines: 
\textbf{CAPQL}~\citep{capql}, \textbf{MORL/D}~\cite{morl_decomposition}, \textbf{PG-MORL}~\citep{pgmorl}, \textbf{PreCo}~\cite{preco}, \textbf{C-MORL}~\citep{cmorl}. GPI-LS~\cite{gpi-ls} was unable to complete any environment in the budget of 5 days, so was excluded from this table. 
For discrete-action environments, the number of environment interactions was $2 \times 10^6$ steps.
For the all continuous control environments, we have used $5 \times 10^6$ steps. 
We measure performance with Hypervolume (HV), Expected Utility (EU) and Sparsity (SP). Sparsity has been identified as a noisy and potentially misleading metric in practice~\cite{felten_toolkit_2023}, particularly when a method discovers a larger hypervolume, causing solutions to be more spread out by construction. We report SP for comparability with prior work but treat HV and EU as primary indicators of front quality.
We note the discrepancies between our baseline evaluations and the original C-MORL and PG-MORL papers. First, they evaluate on older \texttt{v4} environments rather than the updated \texttt{v5} suite. Second, their reported sample efficiencies do not account for parallelization; by failing to multiply their stated timestep budget by the 6--8 parallel processes used during training, their true environmental interaction cost is 6$\times$ to 8$\times$ higher than explicitly claimed. We have ensured that the total number of interactions for all baselines and \toolname{} are exactly the same and have provided the code in the supplementary material.


\textbf{\toolname{} achieves State-of-the-art Pareto Front Coverage.} The results in Table~\ref{tab:continuous}, \ref{tab:discrete} and Figure~\ref{fig:pareto-fronts} show that \toolname{} finds the most dominant and complete solution sets, providing a strong leap over the current SOTA baselines. 
Visually, the Pareto fronts in Figure~\ref{fig:pareto-fronts} show \toolname{} (red) discovering solutions that envelop the baselines. In all cases, \toolname{} finds densely packed solutions at extremes in the trade-off space that all methods miss. 
Quantitatively, \toolname{} discovered pareto front has a significantly higher hypervolume and the expected utility is much better than the existing baselines. In \textbf{Ant-2d}, \toolname{} has a higher sparsity metric due to the larger hypervolume (2x gain compared to best baseline). This increase in hypervolume contributes to more distant solutions which results in higher sparsity. In all other cases, the sparsity is much lower, showing extremely dense coverage of the pareto front. 
This superior coverage stems from our core methodological contributions.
By computing a vectorized, per-objective advantage and using decomposed policy gradients, \toolname{} mitigates the destructive gradient interference problem in MORL.
This process preserves more stable per-objective credit assignment signals, boosting the policy's ability to better exploit the reward landscape and learn a broader range of trade-offs. 


\textbf{Diversity Regularization reduces Mode Collapse.} Since we observe high advantage cancellation rates and frequently conflicting gradients, we study how the regularizer provides pressure to prevent collapse. As observed in Figure~\ref{fig:pareto-fronts}, \toolname{} discovers multiple solutions and shows no evidence of collapse in the evaluated settings while other baselines like \textbf{CAPQL}, \textbf{PG-MORL} and \textbf{GPI-LS} collapse. This is particularly visible in Ant-2D and Humanoid-2D, where the methods have only discovered 1-2 pareto optimal solutions.  

The most direct evidence is in the MO-Humanoid-2d results (Table~\ref{tab:continuous}), where several baselines report a Sparsity (SP) of $0$.
This indicates a total collapse to a single dominant policy.
In contrast, \toolname{} achieves a low but non-zero SP ($0.003 \times 10^4$), demonstrating that it has learned a diverse and well-distributed set of policies across the front.
The visual results in Figure~\ref{fig:pareto-fronts} further confirm that \toolname{} discovers rich, well-spaced pareto fronts.

Diverse policies are primarily due to our proposed scaled diversity regularization.
As shown in our ablation study (Table~\ref{tab:ablations}), removing the diversity loss (\toolname{}-DR) results in a clear performance drop and, in some cases, collapse to a single-point front (e.g., Humanoid-2d).
This highlights that explicitly encouraging the policy to produce distinct behaviors for distinct preferences is critical for discovering a complete and useful Pareto front.


\textbf{Ablations.}
We conducted ablation experiments to understand the impact of our changes - Late Stage Weighting \textbf{(LSW)} and Diversity-Regularization \textbf{(DR)}. First, we remove \textbf{LSW} by multiplying the preference weights with the advantages after rollout collection, thereby collecting the weighted advantages instead of the unweighted advantages (in effect, ES).
In this experiment, we do not remove the diversity loss.
Second, we turn off the diversity loss and keep the original decomposed gradient function.

Table~\ref{tab:ablations} shows that both additions are necessary for \toolname{}'s success.
Turning off LSW (column 2), makes the performance suffer considerably.
This shows that learning accurate unweighted returns is necessary to drive correct gradient updates.
When we turn on \textbf{LSW} and turn off \textbf{DDPO} (column 3), we see that the performance improves significantly but it still does not fully approximate the whole front.
In both cases, the policies converged prematurely to a single point front in the Humanoid environment. For Hopper and Ant the combination of low HV, EU and SP values shows that they discovered an inferior Pareto front compared to \toolname{}. 
These experiments show that both innovations are necessary to learn robust policies that approximate a high quality Pareto Front in the single-policy MORL setting. Further ablations with $\lambda_{div}, \alpha$ and noise scale show \toolname{}'s resistance to hyperparameters. 

%% file: table/1-monolithic.tex
\begin{table*}[!h]
\centering
\caption{Performance comparison on \textbf{continuous} environments (Hopper, Ant, Humanoid, Building-9d). Metrics: Hypervolume (HV), Expected Utility (EU) and Sparsity (SP). \textit{T/O} indicates timeout after 5 days.}
\small
\setlength{\tabcolsep}{3pt}
\resizebox{\textwidth}{!}{%
\begin{tabular}{@{}llllllll@{}}
\toprule
\textbf{Environment}
& \textbf{Metrics}
& \textbf{CAPQL}
& \textbf{MORL/D}
& \textbf{PG-MORL}
& \textbf{PreCO}
& \textbf{C-MORL}
& \textbf{\toolname{}}
\\ \midrule

 \multirow{4}{*}{\textbf{Hopper-2d}}
  & HV ($10^5 \uparrow$) & $1.15 \pm 0.08$ & $0.33 \pm 0.00$ & $1.20 \pm 0.09$ & $1.19 \pm 0.10$ & $1.18 \pm 0.05$ & $\mathbf{1.35 \pm 0.08}$ \\
  & EU ($10^2 \uparrow$) & $2.28 \pm 0.07$ & $0.82 \pm 0.01$ & $2.34 \pm 0.10$ & $2.33 \pm 0.10$ & $2.31 \pm 0.05$ & $\mathbf{2.48 \pm 0.081}$ \\
  & SP ($10^2 \downarrow$) & $0.46 \pm 0.10$ & - & $5.13 \pm 5.81$ & $0.49 \pm 0.37$ & $0.92 \pm 0.43$ & $\mathbf{0.11 \pm 0.07}$ \\
\hline

\multirow{4}{*}{\textbf{Hopper-3d}} 
& HV ($10^7 \uparrow$) & $1.62 \pm 0.17$ & $0.59 \pm 0.01$ & $0.63 \pm 0.04$ & $1.70 \pm 0.29$ & $1.62 \pm 0.13$ & $\mathbf{2.20 \pm 0.05}$ \\
& EU ($10^2 \uparrow$) & $1.58 \pm 0.07$ & $0.81 \pm 0.01$ & $0.86 \pm 0.04$ & $1.62 \pm 0.10$ & $1.51 \pm 0.04$ & $\mathbf{1.76 \pm 0.02}$ \\
& SP ($10^2 \downarrow$) & $6.80 \pm 2.38$ & $\mathbf{0.04 \pm 0.05}$ & $0.44 \pm 0.63$ & $0.74 \pm 1.22$ & $0.32 \pm 0.28$ & $0.11 \pm 0.04$ \\
\hline

 \multirow{4}{*}{\textbf{Ant-2d}}
  & HV ($10^5 \uparrow$) & $1.11 \pm 0.69$ & $0.31 \pm 0.00$ & $0.35 \pm 0.08$ & $1.17 \pm 0.25$ & $0.59 \pm 0.02$ & $\mathbf{2.88 \pm 0.03}$ \\
  & EU ($10^2 \uparrow$) & $2.16 \pm 0.94$ & $0.76 \pm 0.01$ & $0.81 \pm 0.23$ & $2.14 \pm 0.19$ & $1.34 \pm 0.03$ & $\mathbf{3.98 \pm 0.34}$ \\
  & SP ($10^3 \downarrow$) & $0.18 \pm 0.07$ & - & $2.20 \pm 3.48$ & $3.61 \pm 2.13$ & $\mathbf{0.13 \pm 0.06}$ & $0.75 \pm 0.36$ \\
\hline

\multirow{4}{*}{\textbf{Ant-3d}} 
& HV ($10^7 \uparrow$) & $1.08 \pm 0.16$ & $0.55 \pm 0.01$ & $0.29 \pm 0.19$ & $0.55 \pm 0.81$ & $0.99 \pm 0.04$ & $\mathbf{3.65 \pm 0.35}$ \\
& EU ($10^2 \uparrow$) & $1.36 \pm 0.23$ & $0.76 \pm 0.01$ & $0.59 \pm 0.17$ & $1.20 \pm 0.20$ & $1.09 \pm 0.02$ & $\mathbf{2.22 \pm 0.11}$ \\
& SP ($10^3 \downarrow$) & $0.58 \pm 0.30$ & $0.07 \pm 0.02$ & $14.06 \pm 19.89$ & $1.96 \pm 0.79$ & $0.01 \pm 0.00$ & $\mathbf{0.005 \pm 0.004}$ \\
\hline

\multirow{4}{*}{\textbf{Humanoid-2d}} 
& HV ($10^5 \uparrow$) & $1.69 \pm 0.21$ & $0.43 \pm 0.00$ & $0.00 \pm 0.00$ & $1.98 \pm 0.02$ & $2.08 \pm 0.09$ & $\mathbf{4.03 \pm 0.05}$ \\
& EU ($10^2 \uparrow$) & $3.11 \pm 0.26$ & $1.08 \pm 0.01$ & $-0.47 \pm 0.04$ & $3.67 \pm 0.02$ & $3.54 \pm 0.12$ & $\mathbf{5.56 \pm 0.18}$ \\
& SP ($10^4 \downarrow$) & $\mathit{0.00 \pm 0.00}$ & $\mathit{0.00 \pm 0.00}$ & $0.49 \pm 0.21$ & $\mathit{0.00\pm 0.00}$ & $0.65 \pm 0.81$ & $\mathbf{0.003 \pm 0.002}$ \\
\hline

\multirow{4}{*}{\textbf{Building-9d}} 
& HV ($10^{31} \uparrow$) & $4.29 \pm 0.73$ & \textit{T/O} & \textit{T/O} & \textit{T/O} & $6.38 \pm 0.08$ & $\mathbf{7.77 \pm 0.11}$ \\
& EU ($10^3 \uparrow$) & $3.31 \pm 0.06$ & \textit{T/O} & \textit{T/O} & \textit{T/O} & $3.44 \pm 0.00$ & $\mathbf{3.50 \pm 0.003}$ \\
& SP ($10^3 \downarrow$) & $4.34 \pm 3.72$ & \textit{T/O} & \textit{T/O} & \textit{T/O} & $1.01 \pm 1.19$ & $\mathbf{0.003 \pm 0.001}$ \\

\bottomrule
\end{tabular}}
\label{tab:continuous}
\end{table*}

%% file: figure/1-fronts.tex
\begin{figure*}[h]
\centering
    \begin{subfigure}{0.195\textwidth}
        \includegraphics[width=\textwidth]{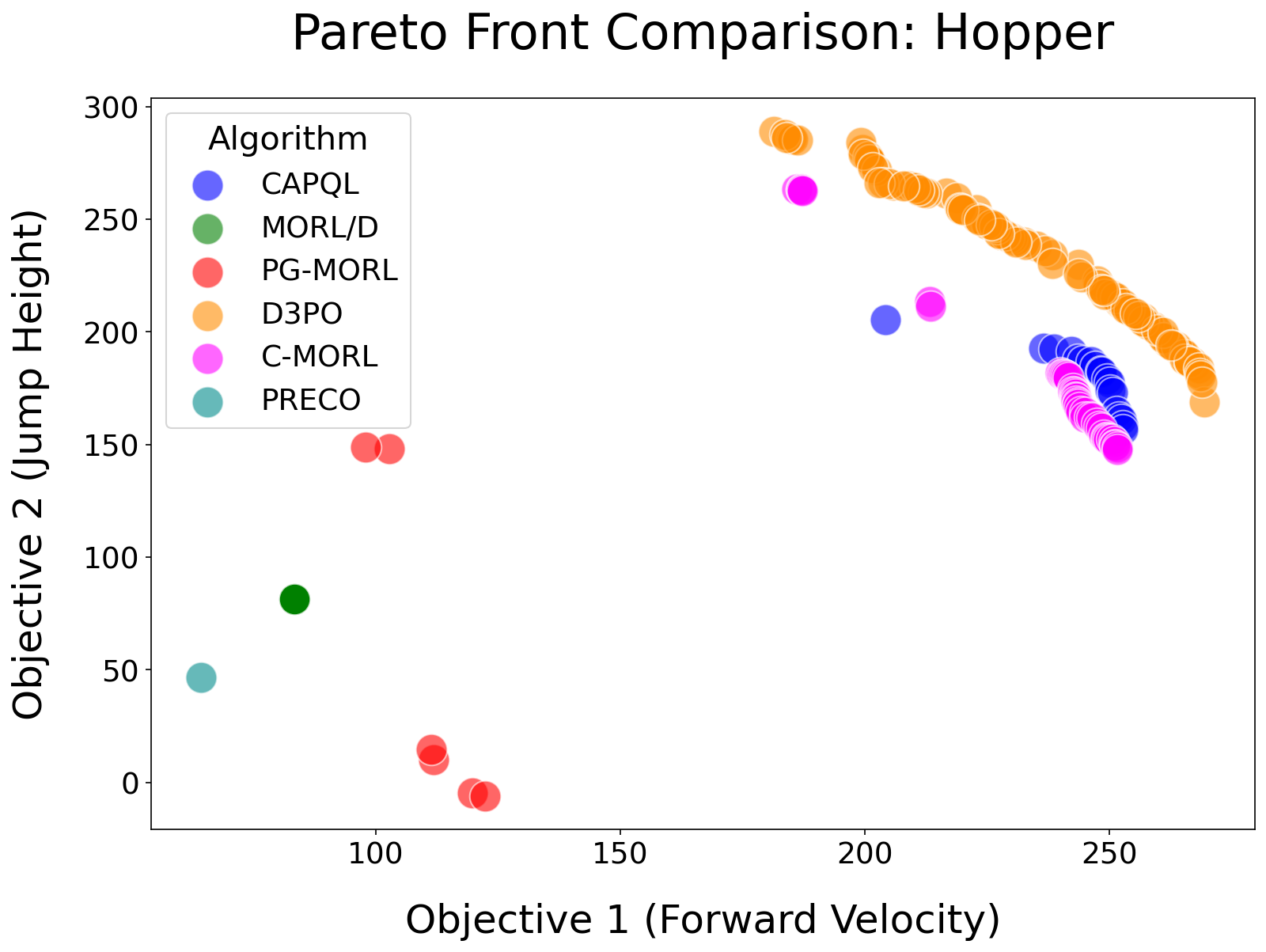}
    \end{subfigure}
    \begin{subfigure}{0.195\textwidth}
        \includegraphics[width=\textwidth]{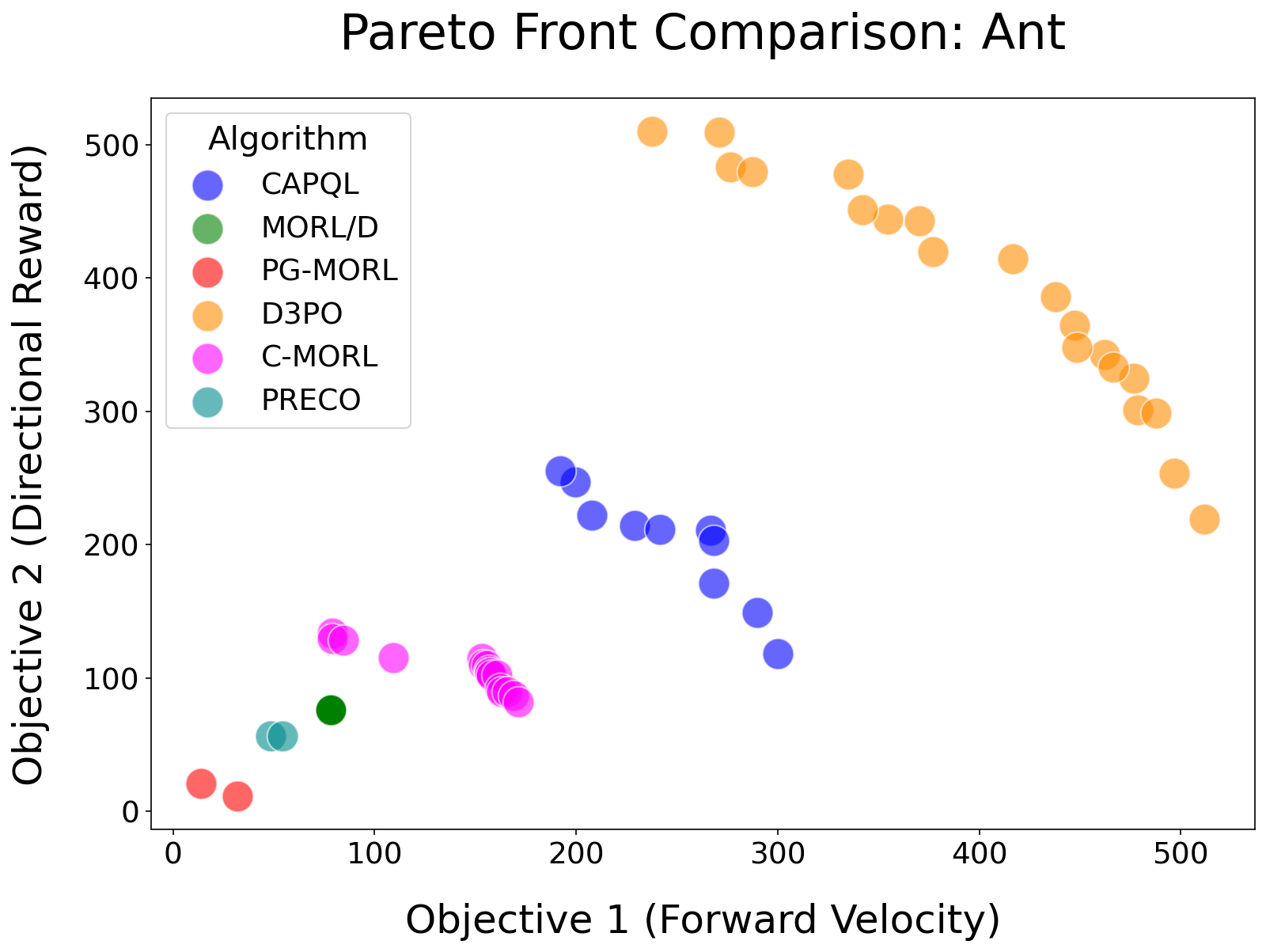}
    \end{subfigure}
    \begin{subfigure}{0.195\textwidth}
        \includegraphics[width=\textwidth]{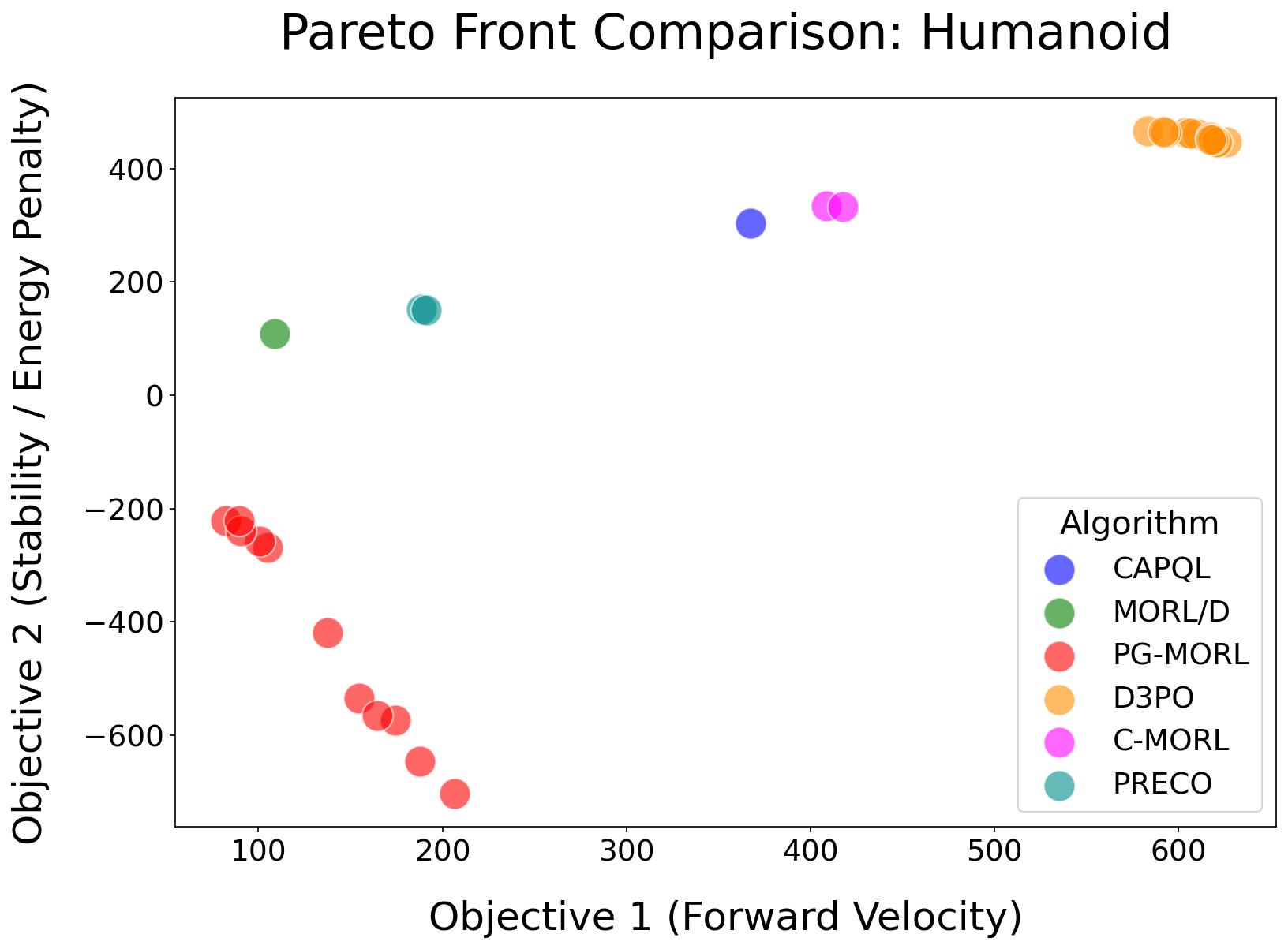}
    \end{subfigure}
    \begin{subfigure}{0.195\textwidth}
        \includegraphics[width=\textwidth]{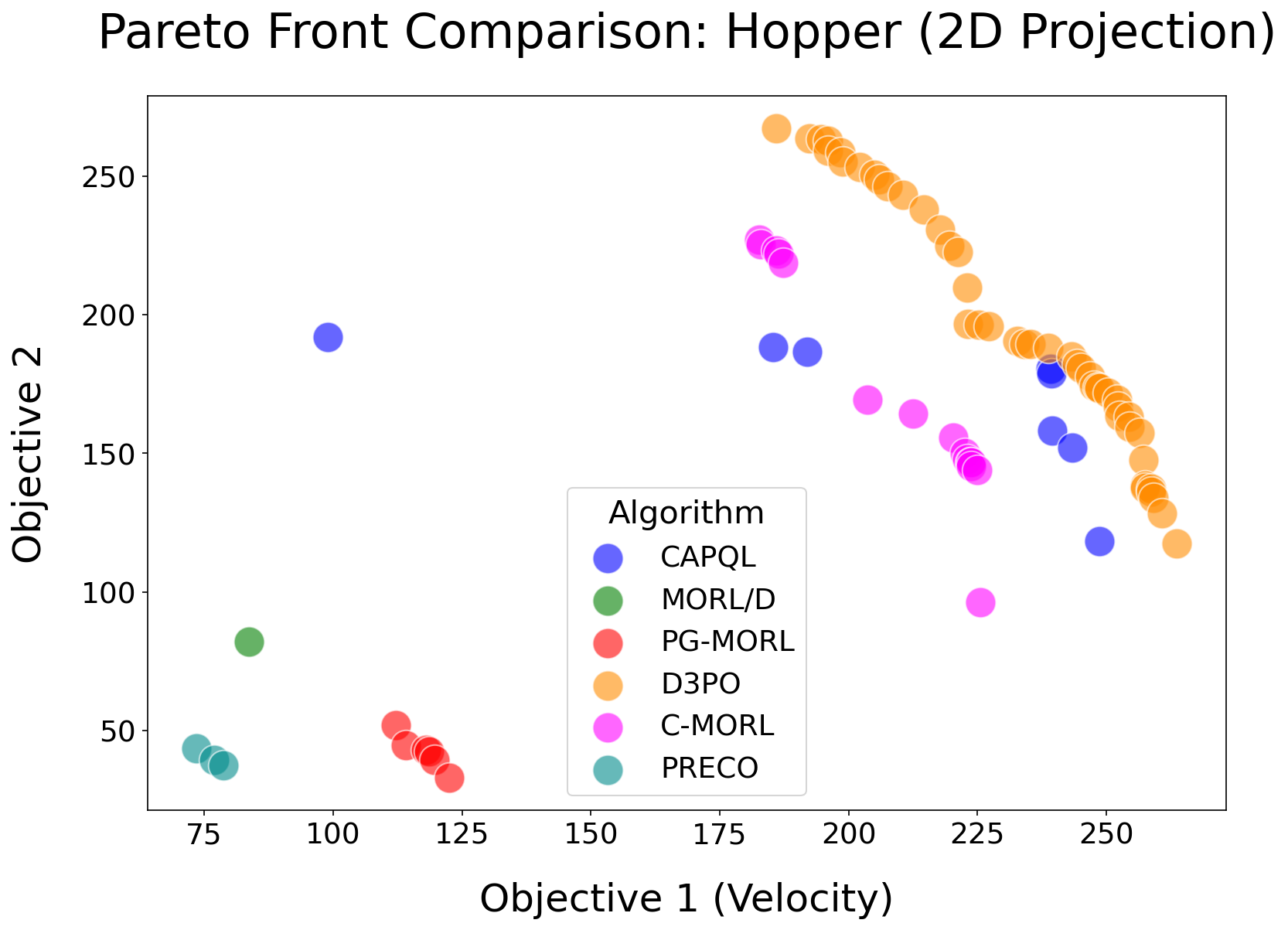}
    \end{subfigure}
    \begin{subfigure}{0.195\textwidth}
        \includegraphics[width=\textwidth]{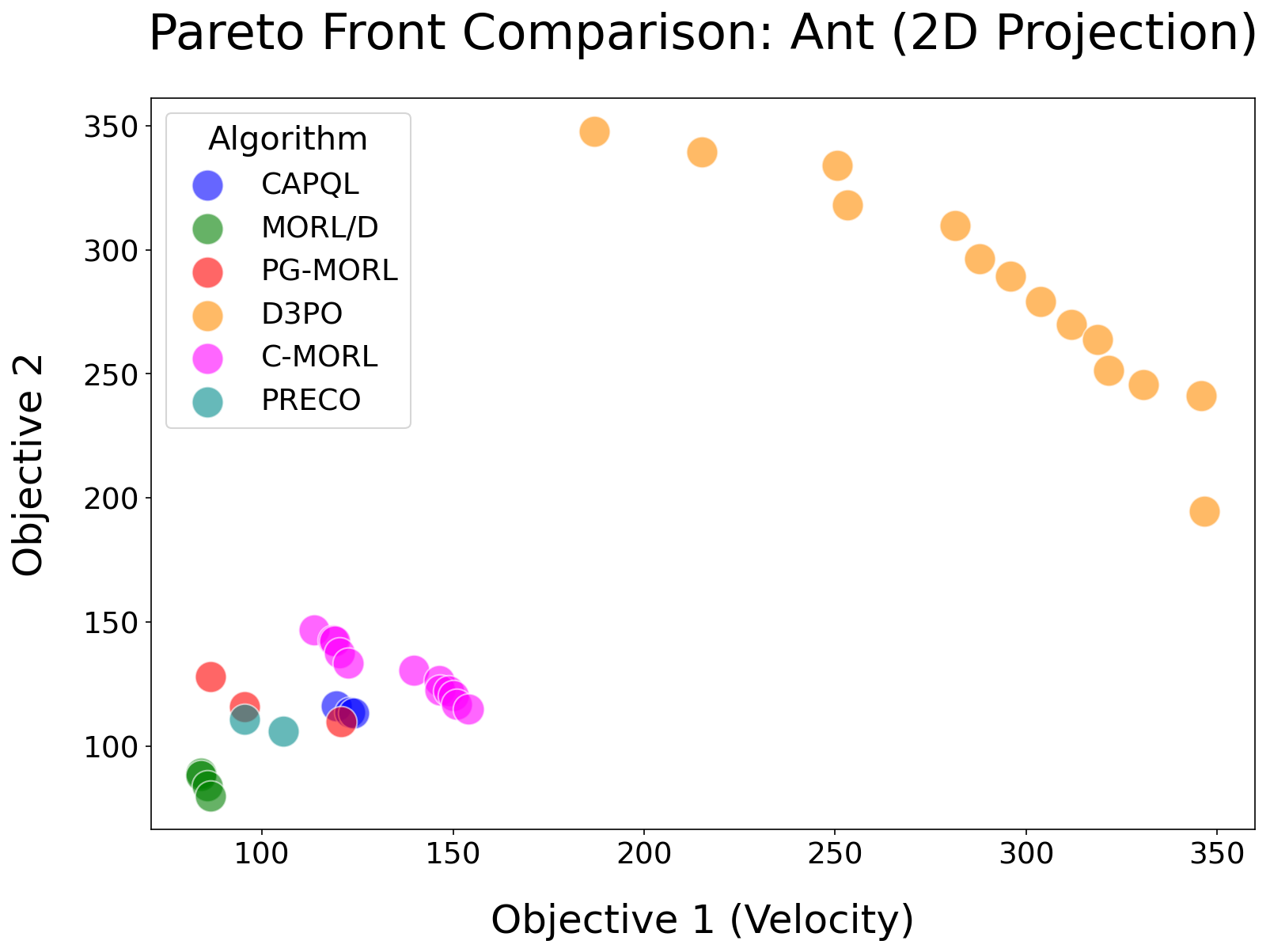}
    \end{subfigure}
   
    \caption{Pareto front comparison on two-objective MO-MuJoCo benchmarks. \toolname{} (red) discovers a uniform and well-distributed front across the trade-off space, whereas C-MORL (blue) refines extreme points at the cost of higher sparsity. Compared to CAPQL, GPI-LS, and PG-MORL, \toolname{} achieves broader coverage and reduced collapse, particularly visible in Ant and Humanoid.}
    \label{fig:pareto-fronts}
\end{figure*}

%% file: documentBody/7-conclusion.tex
\section{Conclusion
\draftStatus{Page Budget: 0.5pg
Status: TSA drafted, JED has reviewed (no notes), AV please review}}
\label{sec:conclusion}

In this work, we introduced \toolname{}, a novel algorithm for training a single, generalizable policy for MORL.
We identified two critical challenges that hinder prior preference-conditioned methods: destructive gradient interference and representational mode collapse.
Our proposed framework addresses these issues through a synergy of two principled mechanisms:
a decomposed optimization process that preserves the integrity of per-objective credit assignment, and a scaled diversity regularization term that enforces a robust and high-fidelity mapping from the preference space to the policy manifold.
Our experiments demonstrate that these two targeted additions to PPO are necessary and sufficient to achieve state-of-the-art MORL performance. \toolname{} discovers more complete and higher-quality Pareto fronts than existing methods, with particularly pronounced advantages in complex, high-dimensional control and many-objective scenarios.

%% file: documentBody/A-code.tex
\section{\toolname{} Pseudocode}
\label{appendix:code}

\input{algorithms/MO-PPO}

%% file: algorithms/MO-PPO.tex
\begin{algorithm*}[!htb]
\caption{Decomposed, Diversity-Driven Policy Optimization}
\label{alg:mo-ppo-single}
\begin{algorithmic}[1]

\REQUIRE 
  Actor $\pi_{\theta}(a \mid s, \omega)$, multi-head critic $V_{\phi}(s, \omega) \in \mathbb{R}^d$,
  Optimizers $\mathrm{Opt}_{\theta}, \mathrm{Opt}_{\phi}$, and hyperparameters $\gamma, \lambda, \epsilon, \beta, \lambda_{\text{div}}, \alpha$

\STATE Initialize network parameters $\theta, \phi$ and rollout buffer $\mathcal{D}$
\STATE Sample an initial preference vector $\omega$ from the preference space $\Omega$
\FOR{iteration $=1, 2, \dots$}
    \STATE Clear rollout buffer $\mathcal{D}$
    
    \FOR{$t=1$ \textbf{to} $T$} 
        \STATE Sample action $a_t \sim \pi_{\theta}( \cdot \mid s_t, \omega)$
        \STATE Execute $a_t$ and observe next state $s_{t+1}$, reward vector $\mathbf{r}_t \in \mathbb{R}^d$, and done flag $d_t$
        \STATE Store transition $(s_t, a_t, \mathbf{r}_t, \omega, \log\pi_{\theta}(a_t \mid s_t, \omega))$ in $\mathcal{D}$
        \STATE $s_t \gets s_{t+1}$
        \IF{$d_t$ is True}
            \STATE Reset environment to get new state $s_t$ and resample a new preference vector $\omega \sim \Omega$
        \ENDIF
    \ENDFOR
    
    \STATE Compute unweighted advantages $\mathbf{A}_t = [A_t^{(1)}, \dots, A_t^{(d)}]$ and returns $\mathbf{G}_t$ for all transitions in $\mathcal{D}$ using GAE with $V_\phi$. 
    
    \FOR{epoch $=1$ \textbf{to} $E$} 
        \FOR{each minibatch $\mathcal{B} \subset \mathcal{D}$}
            \STATE Let $(s, a, \mathbf{A}, \mathbf{G}, \omega, \log\pi_{\text{old}})$ be the data in $\mathcal{B}$
            
            \STATE Predict value vector $\mathbf{V}_\phi(s, \omega) = [V_\phi^{(1)}, \dots, V_\phi^{(d)}]$
            \STATE $\mathcal{L}_{\text{critic}} \gets \frac{1}{d} \sum_{i=1}^d \left(V_\phi^{(i)}(s, \omega) - G^{(i)}\right)^2$
            \STATE Update critic parameters $\phi$ using $\mathrm{Opt}_{\phi}$ and $\nabla_\phi \mathcal{L}_{\text{critic}}$

            \STATE Sample distractor weights $\omega'$ by perturbing and re-normalizing $\omega$
            \STATE Compute per-objective PPO losses $\{\mathcal{L}_{\text{clip}}^{(i)}\}_{i=1}^d$ using unweighted advantages $\mathbf{A}$
            \STATE Compute diversity loss $\mathcal{L}_{\text{diversity}}(\theta) = \mathbb{E}_{s \in \mathcal{B}}\Big[ \big( D_{KL}(\pi_\theta(\cdot \mid s, \omega) \| \pi_\theta(\cdot \mid s, \omega')) - \alpha \|\omega - \omega'\|_1 \big)^2 \Big]$
            \STATE Compute entropy bonus $\mathcal{H} \gets \mathbb{E}_{s \in \mathcal{B}}[\text{H}(\pi_\theta(\cdot \mid s, \omega))]$
            
            \STATE $\mathcal{L}_{\text{actor}} \gets -\left( \sum_{i=1}^d \omega_i \mathcal{L}_{\text{clip}}^{(i)} \right) - \beta \mathcal{H} + \lambda_{\text{div}} \mathcal{L}_{\text{diversity}}$
            \STATE Update actor parameters $\theta$ using $\mathrm{Opt}_{\theta}$ and $\nabla_\theta \mathcal{L}_{\text{actor}}$
        \ENDFOR
    \ENDFOR
\ENDFOR
\end{algorithmic}
\end{algorithm*}

%% file: documentBody/B-metrics.tex
\section{Metrics Definitions}

\begin{definition}[Hypervolume Indicator]
Given a reference point $r \in \mathbb{R}^d$ that all Pareto-optimal returns dominate, the \emph{hypervolume} of a finite set $\{u^k\}$ is, where \textit{LM} stands for Lebesgue Measure:
\[
\mathrm{HV}(\{u^k\}; r) = \text{LM}\left(\bigcup_k \{u \in \mathbb{R}^d : r \leq u \leq u^k\}\right)
\]
\end{definition}

\begin{definition}[Sparsity Indicator]
Let $\{u^1, \dots, u^K\} \subset \mathbb{R}^d$ be an ordered set of Pareto-approximated points.
Define the \emph{sparsity} as:
\[
\mathrm{SP}(\{u^k\}) = \frac{1}{K-1} \sum_{k=1}^{K-1} \| u^{(k+1)} - u^{(k)} \|_2
\]
\end{definition}

\begin{definition}[Expected Utility]
Let $\mathcal{W} \subset \mathbb{R}^d$ be a distribution over preference weights and let $\pi_\omega$ denote the policy conditioned on $\omega$.
The \emph{expected utility} is:
\[
\mathrm{EU} = \mathbb{E}_{\omega \sim \mathcal{W}}[\, \omega^\top G^{\pi_\omega} \,].
\]
\end{definition}

\begin{definition}[Compute Time]
The compute time is defined as the time taken by the algorithm to complete its training given the fixed budget of environment interactions. It is calculated as the wall clock time required to complete the entire training pipeline 
\end{definition}

%% file: documentBody/C-discrete.tex
\section{Discrete Environments Results}

\input{table/1-discrete}

%% file: table/1-discrete.tex
\begin{table*}[!htb]
\centering
\caption{Performance comparison on \textbf{discrete} environments (Minecart, Lunar Lander-4d). Metrics: Hypervolume (HV), Expected Utility (EU), Sparsity (SP).}
\label{tab:discrete}
\small
\begin{tabular}{@{}lllllll@{}}
\toprule
\textbf{Environment}
& \textbf{Metrics}
& \textbf{GPI-LS}
& \textbf{C-MORL}
& \textbf{\toolname{}}
\\ \midrule

\multirow{3}{*}{\textbf{Minecart}} 
& HV ($10^2 \uparrow$)
& $5.05 \pm 0.37$
& $4.65 \pm 0.27$
& $5.35 \pm 0.20$
\\
& EU ($10^{-1} \uparrow$)
& $\mathbf{2.75 \pm 0.17}$
& $0.50 \pm 0.40$
& $\mathbf{3.00 \pm 0.10}$
\\
& SP ($10^{-1} \downarrow$)
& $0.20 \pm 0.05$
& $0.50 \pm 0.60$
& $\mathbf{0.00 \pm 0.00}$
\\ \hline

\multirow{3}{*}{\textbf{Lunar Lander-4d}} 
& HV ($10^9 \uparrow$)
& $1.06 \pm 0.16$
& $0.88 \pm 0.02$
& $\mathbf{1.23 \pm 0.04}$
\\
& EU ($10^1 \uparrow$)
& $1.81 \pm 0.34$
& $1.23 \pm 0.18$
& $\mathbf{2.5 \pm 0.57}$
\\
& SP ($10^3 \downarrow$)
& $0.13 \pm 0.01$
& $0.14 \pm 0.11$
& $0.21 \pm 0.13$ 
\\ \hline
\bottomrule
\end{tabular}
\end{table*}

%% file: documentBody/D-Ablation.tex
\section{Ablation Experiments}
\label{sec:lambda_div}

\input{table/2-humanoid}

\input{table/3-lambda}

Table~\ref{tab:humanoid_ablation} reports ablation results on Humanoid-2d across a sweep of $\lambda_{\text{div}}$ values. \textbf{All ablations were run with a smaller budget of $1.5 \times 10^6$ steps.} The purpose of the ablation budget is to isolate the contribution of each component under the same reduced budget, not to reproduce Table~\ref{tab:continuous} numbers. 
These results demonstrate that the diversity regularizer itself plays a critical role in shaping the discovered Pareto front. 
Without diversity encouragement ($\lambda_{\text{div}}=0$), the algorithm collapses toward limited modes, yielding weaker hypervolume and expected utility despite producing seemingly low sparsity values.
Introducing a nonzero regularizer ($\lambda_{\text{div}}>0$) resolves this issue by reducing mode collapse and maintaining broad front coverage, thereby producing substantially stronger Pareto sets.

At the same time, the quantitative metrics reveal that the performance is relatively insensitive to the precise choice of $\lambda_{\text{div}}$. 
Across the range $\lambda_{\text{div}} \in \{0.01, 0.1, 0.5, 1.0\}$, hypervolume and expected utility remain consistently high, and sparsity values remain comparable. 
This indicates that while the presence of the diversity term is essential, its specific scaling does not heavily influence the outcome. 
Overall, these ablations reinforce that the diversity regularizer is the key mechanism enabling robust front discovery, and that the method is not fragile to the exact tuning of $\lambda_{\text{div}}$.

\input{table/4-alpha}

Table~\ref{tab:alpha_ablation} reports similar results. When $\alpha = 0$, the weights scaling parameter is turned off. This keeps the KL term active, and the loss function now tries to minimize the KL. By minimizing the KL, the function actively promotes collapse. Thus, $\alpha$ is an extremely important parameter. When $\alpha = 0.1$ and $\alpha = 1$, the results are similar. This shows that \toolname{} is robust to the values of the weight parameter. Choosing a very high value $\alpha=10$ is also detrimental to performance, as that term dominates the loss function. Thus, a reasonable choice for $\alpha$ is between 0.1 and 1.

%% file: table/2-humanoid.tex
\begin{table*}
\centering
\caption{Ablation results showing the contributions of Late Stage Weighting (LSW) and Diversity-Driven Policy Optimization (DDPO) in \toolname{}.
LSW improves stability but often collapses the Pareto front (SP = 0), while DDPO preserves diversity and yields more uniform fronts. The full \toolname{} consistently achieves the best trade-off across HV, EU, and SP.
}
\label{tab:ablations}
\begin{tabular}{@{}lllll@{}}
\toprule
\textbf{Environment} & \textbf{Metrics} & \textbf{\toolname{}} & \toolname{}$\setminus$\textbf{LSW} & \toolname{}$\setminus$\textbf{DDPO} \\ \midrule

\multirow{3}{*}{Humanoid-2d} 
& HV ($10^5 \uparrow$)     & $\mathbf{3.76 \pm 0.11}$ & $1.50 \pm 0.17$ & $2.32 \pm 0.05$ \\ 
& EU ($10^{2} \uparrow$)  & $\mathbf{5.11 \pm 0.09}$ & $2.87 \pm 0.22$ & $3.83 \pm 0.05$ \\
& SP ($10^{4} \downarrow$) & $\mathbf{0.003 \pm 0.001}$ & $0^*$ & $0^*$ \\ \midrule

\multirow{3}{*}{Hopper-2d} 
& HV ($10^5 \uparrow$)     & $\mathbf{1.30 \pm 0.03}$ & $1.23 \pm 0.03$ & $1.22 \pm 0.06$ \\ 
& EU ($10^2 \uparrow$)     & $\mathbf{2.47 \pm 0.01}$ & $2.38 \pm 0.05$ & $2.42 \pm 0.05$ \\
& SP ($10^2 \downarrow$)   & $0.26 \pm 0.31$ & $0.08 \pm 0.02$ & $\mathbf{0.04 \pm 0.02}$ \\ \midrule

\multirow{3}{*}{Ant-2d} 
& HV ($10^5 \uparrow$)     & $\mathbf{1.91 \pm 0.18}$ & $1.53 \pm 0.11$ & $1.86 \pm 0.07$ \\ 
& EU ($10^2 \uparrow$)     & $\mathbf{3.14 \pm 0.21}$ & $2.71 \pm 0.13$ & $3.09 \pm 0.06$ \\
& SP ($10^3 \downarrow$)   & $0.66 \pm 0.40$ & $\mathbf{0.18 \pm 0.07}$ & $0.36 \pm 0.09$ \\ 

\bottomrule
\end{tabular}
\end{table*}

%% file: table/3-lambda.tex
\begin{table}[H]
\centering

\caption{Ablation results on MO-Humanoid-2d across different values of $\lambda_{\text{div}}$. 
The results show that the discovered Pareto front remains stable and high-performing over a wide range of $\lambda_{\text{div}}$, indicating robustness of the method to this hyperparameter.}

\label{tab:humanoid_ablation}
\begin{tabular}{@{}lccccc@{}}
\toprule
\textbf{Metric} & $\lambda_{\text{div}}=0$ & $\lambda_{\text{div}}=0.01$ & $\lambda_{\text{div}}=0.1$ & $\lambda_{\text{div}}=0.5$ & $\lambda_{\text{div}}=1.0$ \\ \midrule

HV ($10^5 \uparrow$)   & $2.32 \pm 0.05$ & $\mathbf{3.76 \pm 0.11}$ & $3.73 \pm 0.07$ & $3.72 \pm 0.10$ & $3.73 \pm 0.07$ \\ 
EU ($10^{2} \uparrow$) & $3.83 \pm 0.05$ & $\mathbf{5.11 \pm 0.09}$ & $5.08 \pm 0.06$ & $5.07 \pm 0.09$ & $5.07 \pm 0.06$ \\ 
SP ($10^{3} \downarrow$) & $0^*$ & $\mathbf{0.03 \pm 0.01}$ & $0.047 \pm 0.045$ & $0.059 \pm 0.044$ & $0.053 \pm 0.032$ \\ 

\bottomrule
\end{tabular}
\end{table}

%% file: table/4-alpha.tex
\begin{table}[H]
\centering
\begin{tabular}{@{}lcccc@{}}
\toprule
\textbf{Metric} & $\alpha=0$ & $\alpha=0.1$ & $\alpha=1$ & $\alpha=10$ \\ 
\midrule

HV ($10^{5}\!\uparrow$) 
& $2.50 \pm 0.12$
& $3.71 \pm 0.08$
& $\mathbf{3.76 \pm 0.11}$
& $3.20 \pm 0.10$ \\

EU ($10^{2}\!\uparrow$)
& $3.90 \pm 0.09$
& $5.03 \pm 0.07$
& $\mathbf{5.11 \pm 0.09}$
& $4.80 \pm 0.27$ \\

SP ($10^{3}\!\downarrow$)
& $0^{*}$
& $0.07 \pm 0.02$
& $\mathbf{0.03 \pm 0.01}$
& $0.12 \pm 0.08$ \\

\bottomrule
\end{tabular}
\caption{Ablation results on MO-Humanoid-2d across different values of $\alpha$.}
\label{tab:alpha_ablation}
\end{table}

%% file: documentBody/E-gradient_analysis.tex
\section{A Gradient-Level Analysis of Scalarization in Multi-Objective RL}
\label{sec:scalarization_analysis}

A central design choice in any multi-objective reinforcement learning algorithm is how the vector-valued objective $\vec{J}(\pi) \in \mathbb{R}^d$ is reduced to a scalar signal suitable for gradient-based policy optimization.
Despite the diversity of algorithmic frameworks proposed in recent years, we argue that many existing methods, whether framed as population-based search, constrained optimization, or preference-conditioned learning, fundamentally rely on \emph{linear scalarization} as their core optimization primitive.
In this section, we provide a precise, gradient-level classification of existing approaches, identify the structural limitations that linear scalarization imposes when combined with trust-region methods like PPO, and explain how \toolname{} addresses these limitations without abandoning the linear paradigm.

\subsection{Linear Scalarization as the Universal Primitive}

Given a preference vector $\omega \in \Delta^{d-1}$, linear scalarization defines the optimization target as $\max_\pi \; \omega^\top \vec{J}(\pi)$. This formulation is attractive for two reasons: (i) it is compatible with standard single-objective RL algorithms, requiring only that rewards be pre-combined as $r_\omega(s,a) = \omega^\top \vec{r}(s,a)$, and (ii) for discounted MDPs with stationary policies, the achievable objective set $\mathcal{J}$ is a convex polytope~\citep{capql}, guaranteeing that every point on the \emph{convex} Pareto front can be recovered by some linear weight vector.

We now examine how each major baseline implements this primitive at the gradient level.

\paragraph{PG-MORL~\citep{pgmorl}.}
PG-MORL maintains a population of separate policies, each trained by standard PPO with a linearly scalarized reward $r_\omega = \omega^\top \vec{r}$.
Its novelty lies in a prediction model that guides evolutionary selection of which weight vectors to train next.
However, once a direction $\omega$ is selected, the PPO surrogate operates on the scalar advantage $A_\omega = \omega^\top \vec{A}$, the definition of \emph{early scalarization} (ES).

\paragraph{CAPQL~\citep{capql}.}
Far from challenging linear scalarization, CAPQL provides its strongest theoretical defense.
The authors prove that for stationary policies, the induced value function range is convex, implying that any Pareto-optimal point is recoverable via linear weights.
Its preference-conditioned vector Q-network $\vec{Q}(s,a,\omega)$ is ultimately optimized through the scalar projection $\omega^\top \vec{Q}(s,a,\omega)$ in the actor loss.
The addition of a strongly concave reward augmentation term stabilizes training but does not alter the linear nature of the optimization target.

\paragraph{GPI-LS~\citep{gpi-ls}.}
The name itself declares the paradigm: Generalized Policy Improvement with \emph{Linear Support}.
GPI-LS trains a discrete set of policies, each on a specific linear scalarization.
Its contribution is a principled prioritization scheme that identifies which weight vector $\omega$ currently yields the worst performance across known policies.
Action selection remains $\arg\max_a\, \omega^\top \vec{Q}(s,a)$.

\paragraph{C-MORL~\citep{cmorl}.}
C-MORL introduces a two-stage framework: a warmup phase training parallel policies on fixed linear weights, followed by an extension phase formulated as constrained optimization (maximize objective $i$ subject to constraints on objectives $j \neq i$).
While the constrained formulation appears to depart from linear scalarization, its practical implementation reveals otherwise.
Constrained policy optimization in deep RL is enforced via Lagrangian relaxation, introducing dual variables $\lambda$ that convert the constrained problem into:
\begin{equation}
    \mathcal{L}(\theta, \lambda) = J_i(\theta) + \sum_{j \neq i} \lambda_j \left( J_j(\theta) - c_j \right)
\end{equation}
This is algebraically identical to linear scalarization with dynamically tuned weights.
When this Lagrangian is translated into a PPO surrogate, the optimizer receives the combined advantage $A_{\text{total}} = A_i + \sum_j \lambda_j A_j$, which is early scalarization by definition.

\subsection{Post-Hoc Gradient Repair: MOPPO and PCRL}

A second class of methods recognizes the gradient conflicts inherent in linear scalarization and attempts to repair them after they occur.

\paragraph{MOPPO~\citep{Terekhov2024InSFA}.}
MOPPO extends PPO to the multi-objective setting with a single preference-conditioned policy.
Per-objective advantages are normalized via PopART, then scalarized as $\hat{A} = \alpha^\top (\hat{\vec{A}} / \sigma)$ before being passed to the PPO clipped surrogate (Equation 19 of~\cite{Terekhov2024InSFA}).
This corresponds to early scalarization in the sense that: the clipping operator sees a single scalar advantage, and conflicting objectives can still cancel before clipping has any effect.

\paragraph{PCRL / PreCo~\citep{preco}.}
PCRL explicitly acknowledges that linear scalarization produces conflicting gradients and applies multi-objective optimization (MOO) techniques, including EPO, CAGrad, and SDMGrad, to manipulate the gradient \emph{after} computing per-objective policy gradients from a linearly scalarized objective.
While the post-processed update direction is nonlinear in the per-objective gradients, the underlying optimization target remains $\omega^\top \vec{v}^{\pi}$.
The gradient surgery is a repair mechanism applied to a linear objective, not a change in the objective itself.
This repair comes at significant computational cost: each policy update requires solving a constrained quadratic program over the gradient vectors.

\subsection{Where \toolname{} Fits: Structural Repair via Late-Stage Weighting}

The analysis above reveals a consistent pattern: among all PPO-based multi-objective methods, every existing approach applies scalarization \emph{before} the clipping operator acts.
Whether the scalarization is a fixed linear combination (PG-MORL, GPI-LS), a Lagrangian-tuned combination (C-MORL), or a variance-normalized combination (MOPPO), the PPO surrogate always receives a single scalar advantage.
When objectives conflict ($A_i > 0$ and $A_j < 0$ for a given action), the scalar advantage $\omega^\top \vec{A} \approx 0$ suppresses the gradient signal, and the clipping operator, calibrated for the magnitude of individual advantages, cannot distinguish genuine near-zero advantages from destructive cancellation.

\toolname{} resolves this by reversing the order of operations.
Late-Stage Weighting (LSW) computes the PPO clipped surrogate \emph{independently for each objective}:
\begin{equation}
    \mathcal{J}_{\text{LSW}}(\theta) = \sum_{i=1}^{d} w_i \, L_i^{\text{CLIP}}(\theta), \quad L_i^{\text{CLIP}} = \mathbb{E}\left[\min\left(\rho\, A_i,\; \text{clip}(\rho, 1{-}\epsilon, 1{+}\epsilon)\, A_i\right)\right]
\end{equation}
The linear combination is applied \emph{after} each objective's advantage has been independently bounded by the trust region.
This preserves per-objective gradient information: when $A_i > 0$ and $A_j < 0$, the clipping operator correctly constrains the positive update for objective $i$ while the negative signal from objective $j$ provides an appropriate penalty.
The result is a component-wise prioritization that is inherently controlled by PPO's trust region, without requiring gradient surgery or nonlinear scalarization.

\toolname{} further addresses the mode collapse risk inherent in single-policy preference-conditioned learning through Scaled Diversity Regularization (SDR), which penalizes the squared difference between a target KL divergence (proportional to preference distance) and the actual KL divergence between action distributions for nearby preference vectors.
This mechanism is absent from all baselines except MOPPO's entropy control, which operates on a fundamentally different axis (global exploration vs.\ preference-conditioned behavioral diversity).

Table~\ref{tab:scalarization_taxonomy} summarizes the landscape.

\begin{table}[t]
\centering
\caption{Taxonomy of scalarization strategies in preference-conditioned MORL.
\emph{ES} = Early Scalarization (scalarize before clipping/update);
\emph{LSW} = Late-Stage Weighting (clip per-objective, then weight).
``Repair'' denotes post-hoc gradient manipulation applied to a linearly scalarized objective.}
\label{tab:scalarization_taxonomy}
\small
\setlength{\tabcolsep}{4pt}
\begin{tabular}{@{}lccccc@{}}
\toprule
\textbf{Method} & \textbf{Objective} & \textbf{Scalarization} & \textbf{Policy} & \textbf{Algorithm} & \textbf{Diversity} \\
\midrule
PG-MORL  & Linear & ES & Many & PPO & --- \\
GPI-LS   & Linear & ES & Many & DQN & --- \\
C-MORL   & Linear (Lagrangian) & ES & Many & PPO & --- \\
CAPQL    & Linear (augmented) & ES & Single & SAC & --- \\
MOPPO    & Linear (PopART) & ES & Single & PPO & Entropy \\
PCRL     & Linear + Repair & ES + Surgery & Single & SAC & Cosine sim. \\
\midrule
\textbf{\toolname{}} & \textbf{Linear} & \textbf{LSW} & \textbf{Single} & \textbf{PPO} & \textbf{KL div.} \\
\bottomrule
\end{tabular}
\end{table}

\toolname{} occupies a unique position in this landscape.
It retains the theoretical guarantees of linear scalarization, meaning any convex Pareto-optimal point is recoverable, while resolving the practical failure mode that undermines every other PPO-based approach.
Unlike PCRL, it requires no per-step quadratic program.
And unlike multi-policy methods (PG-MORL, GPI-LS, C-MORL), it deploys a single network at inference time, with $O(1)$ cost for any preference vector.

\subsection{Empirical Diagnostics: Advantage Cancellation and Cosine Gradient Analysis}
\label{subsec:advantage_cancellation}

To empirically validate the structural limitations of Early Scalarization (ES) discussed above, we track both the magnitude and frequency of gradient conflict across our benchmark suite. Because ES squashes multi-objective rewards into a single scalar before passing them to the PPO surrogate, the optimizer operates on a scalar projection that does not preserve per-objective structure. When an action benefits one objective ($A_i > 0$) but penalizes another ($A_j < 0$), these signals destructively interfere. 

We quantify this failure mode through two complementary metrics: the absolute magnitude of the cancelled signal, and the step-by-step prevalence of gradient opposition.

\textbf{The Magnitude of Signal Loss.} 
Figure~\ref{fig:advantage_cancellation_grid} tracks the advantage cancellation rate throughout training. This metric measures the percentage of the total available advantage magnitude that is reduced due to cancellation in the linear summation step before trust-region clipping occurs. The results reveal that advantage cancellation is not a transient artifact, but a severe, persistent pathology. Across diverse continuous locomotion tasks (e.g., \texttt{mo-ant-v5}) and highly discontinuous discrete fronts (e.g., \texttt{fruit-tree-v0}), ES can reduce up to 50–65\% of the aggregated advantage magnitude at each update. By prematurely homogenizing the advantage, ES reduces the effective signal available to the trust-region optimizer when navigating competing objectives.

\textbf{The Prevalence of Gradient Conflict.} 
Complementing the magnitude of signal loss, Figure~\ref{fig:gradient_prevalence_grid} illustrates the frequency of these conflicts, visualized as the percentage of training steps where the cosine similarity between per-objective gradients is negative. These heatmaps reveal a distinct shift during the learning process. In early training, objectives often align as the agent learns basic environmental competence (e.g., surviving while moving forward). However, the moment the agent reaches the Pareto front, the objectives become more strongly competing, and the prevalence of gradient conflict sharply increases. In high-dimensional environments such as \texttt{mo-humanoid-v5} and \texttt{Building-9d}, opposing gradients between objective pairs frequently plateau between 30\% and 55\% of all training steps. 

Together, these diagnostics show that gradient conflict under Early Scalarization is persistent and widespread across the evaluated environments. At almost every timestep along the Pareto manifold, ES reduces the usable gradient signal available for trade-off discovery. \toolname{} mitigates this bottleneck  by preserving decomposed advantages and applying Late-Stage Weighting (LSW). It safely subjects every conflicting signal to the trust-region operator before any scalarization can induce destructive interference.
\begin{figure*}[!h]
\centering
\begin{subfigure}[b]{0.32\textwidth}
  \centering
  \includegraphics[width=\linewidth]{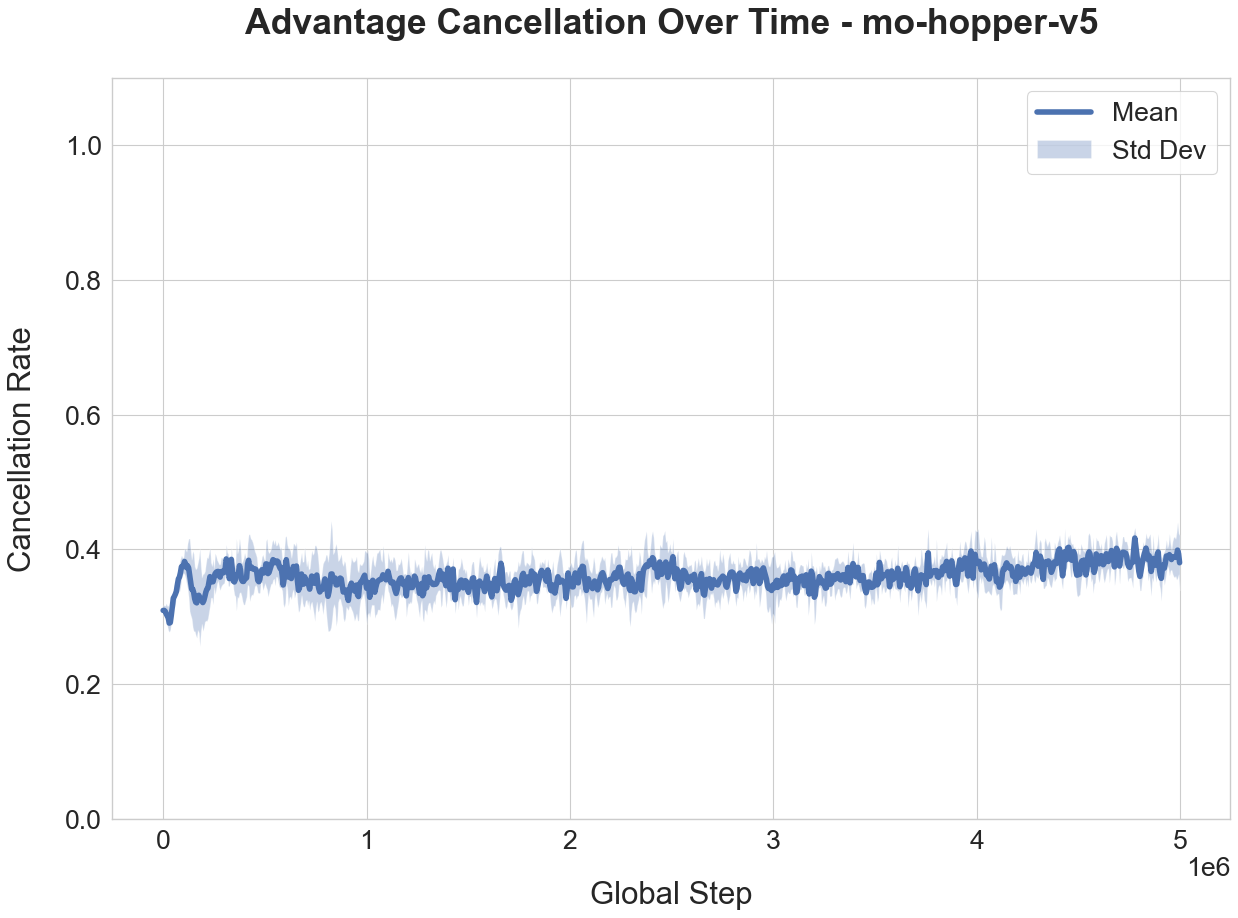}
  \caption{Hopper-3d}
\end{subfigure}\hfill
\begin{subfigure}[b]{0.32\textwidth}
  \centering
  \includegraphics[width=\linewidth]{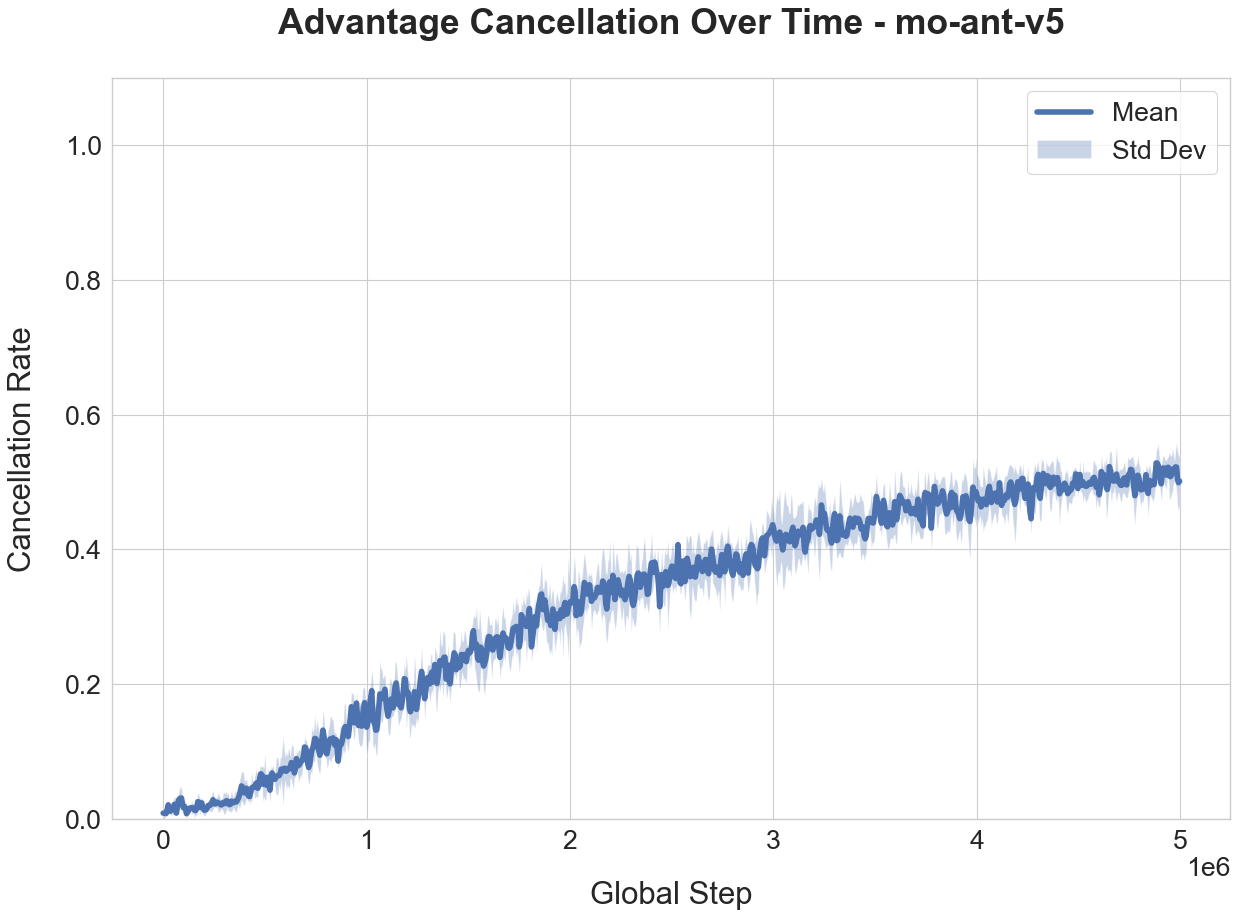}
  \caption{Ant-3d}
\end{subfigure}\hfill
\begin{subfigure}[b]{0.32\textwidth}
  \centering
  \includegraphics[width=\linewidth]{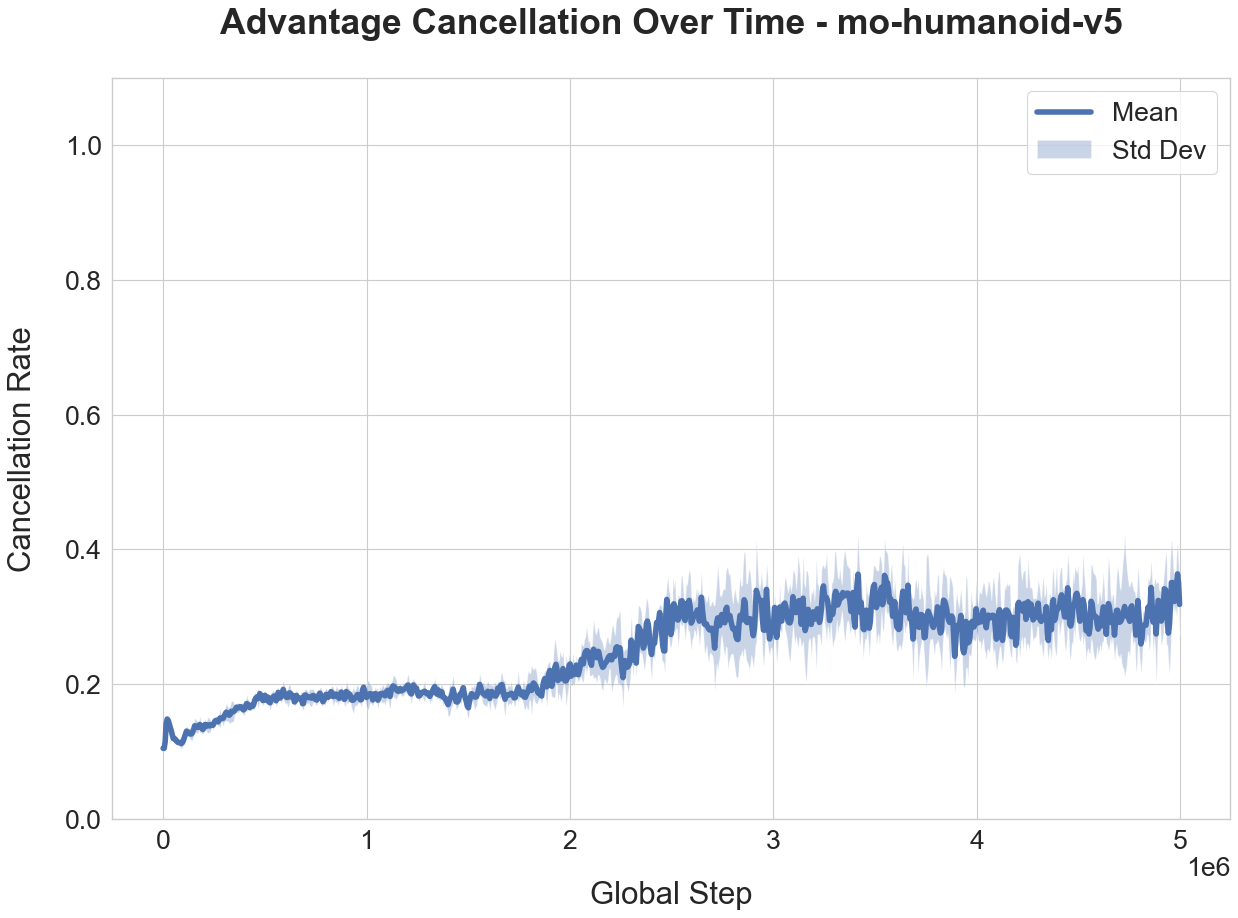}
  \caption{Humanoid-2d}
\end{subfigure}

\vspace{0.3cm}

\begin{subfigure}[b]{0.32\textwidth}
  \centering
  \includegraphics[width=\linewidth]{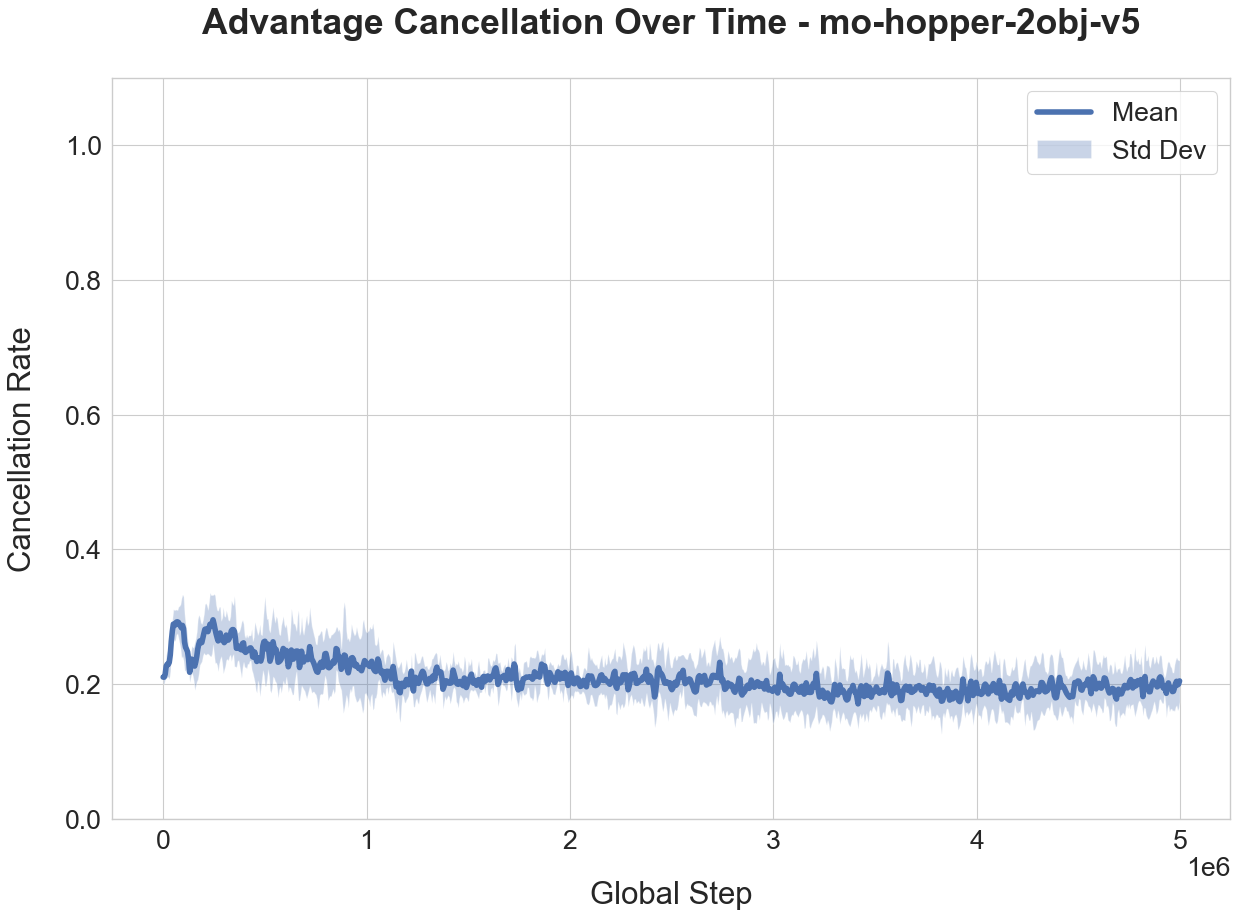}
  \caption{Hopper-2d}
\end{subfigure}\hfill
\begin{subfigure}[b]{0.32\textwidth}
  \centering
  \includegraphics[width=\linewidth]{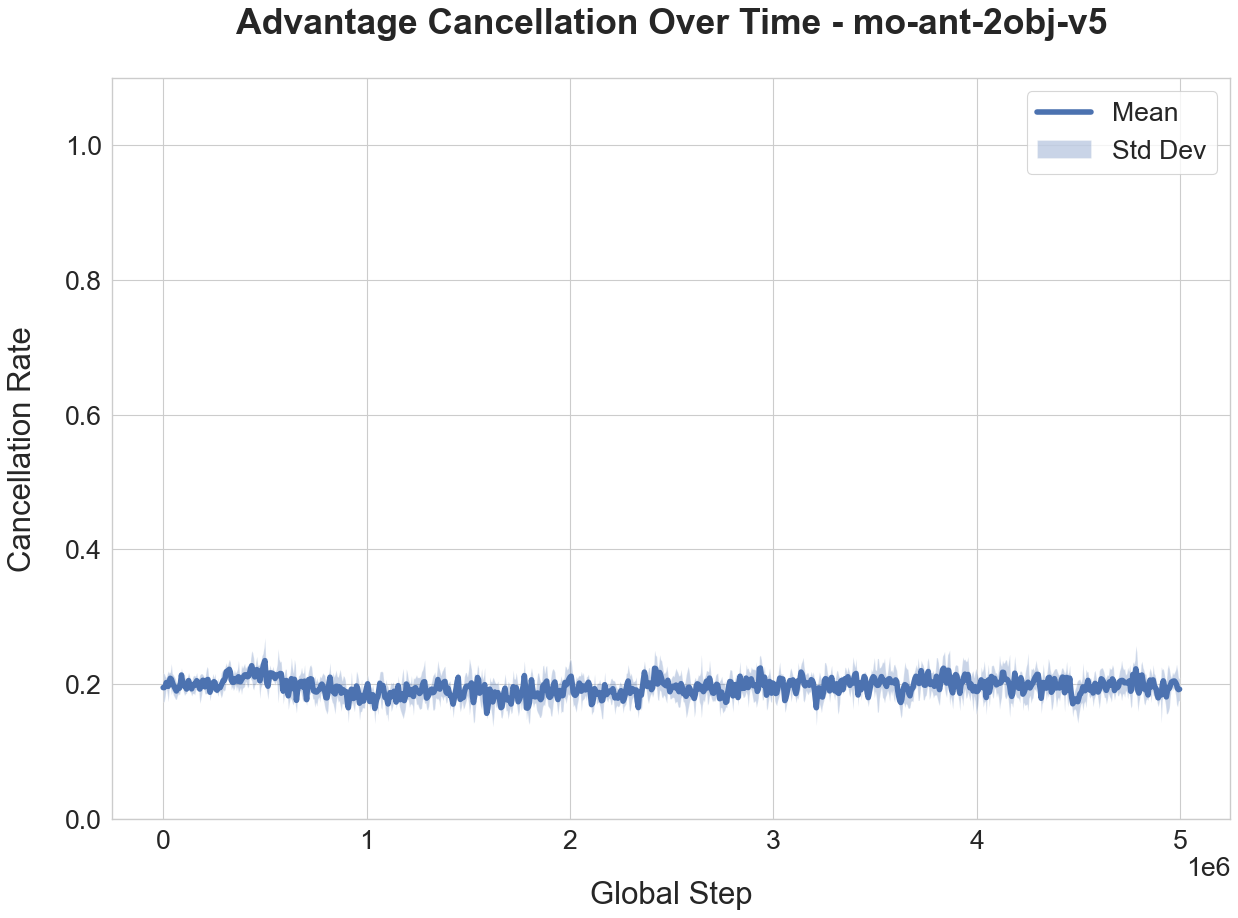}
  \caption{Ant-2d}
\end{subfigure}\hfill
\begin{subfigure}[b]{0.32\textwidth}
  \centering
  \includegraphics[width=\linewidth]{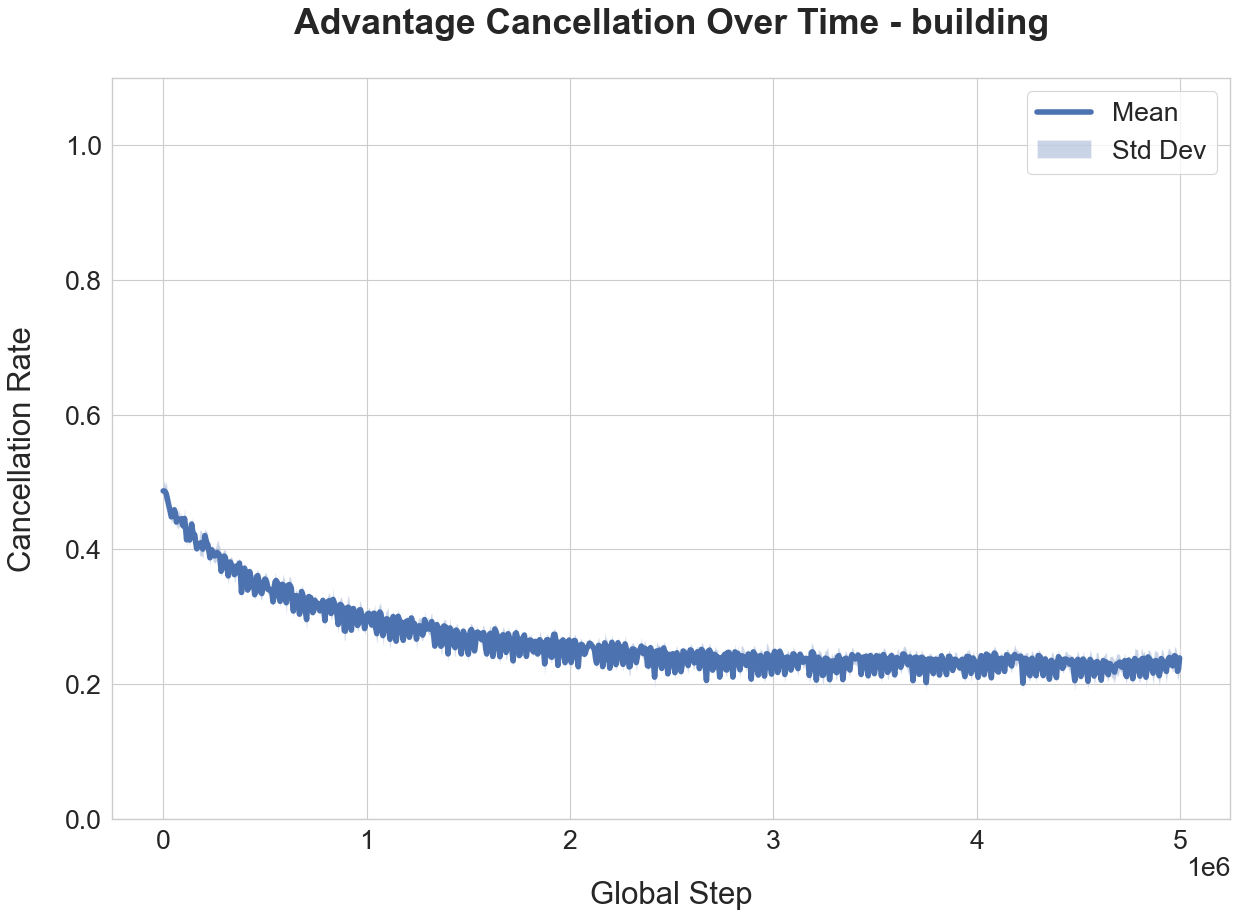}
  \caption{Building-2d}
\end{subfigure}

\vspace{0.3cm}

\begin{subfigure}[b]{0.32\textwidth}
  \centering
  \includegraphics[width=\linewidth]{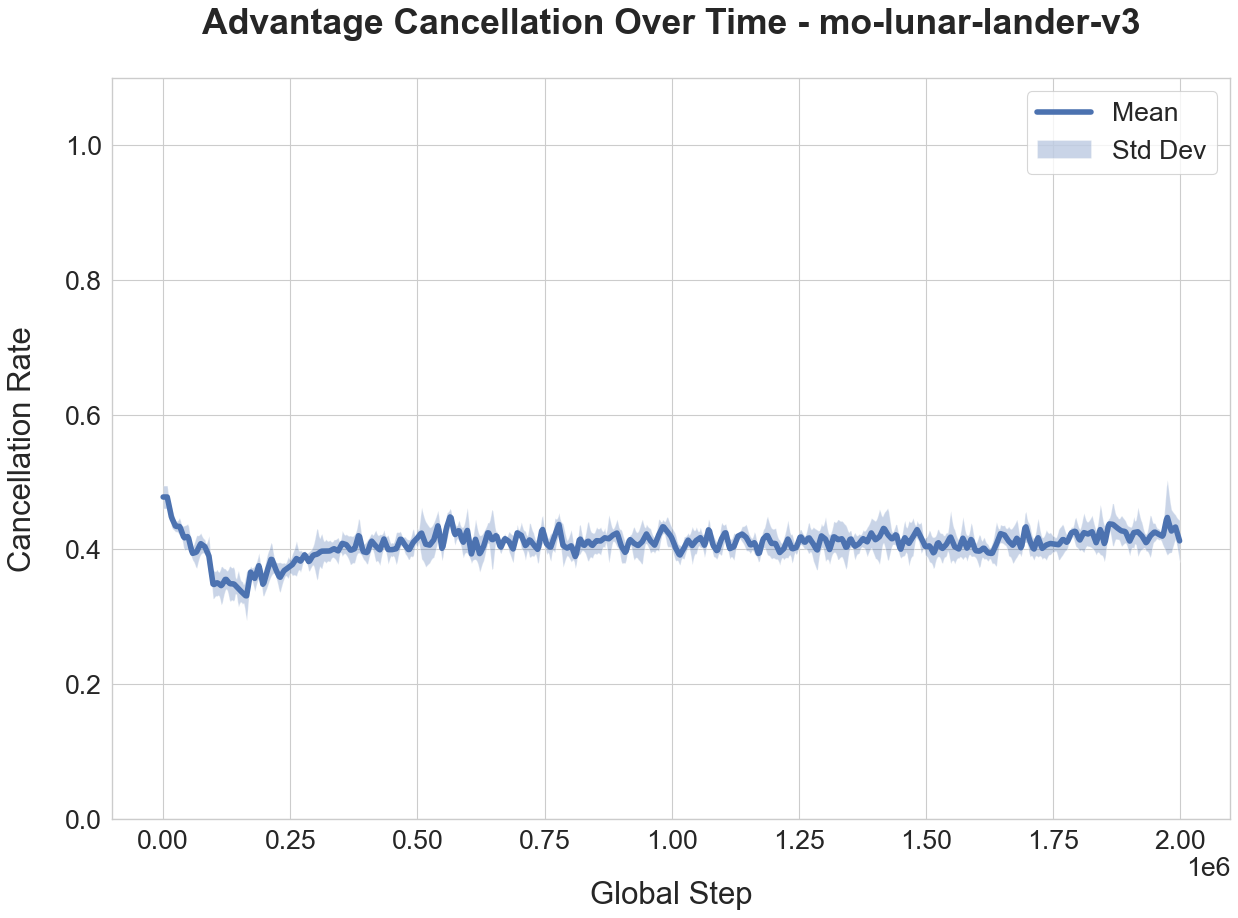}
  \caption{Lunar Lander-4d}
\end{subfigure}\hfill
\begin{subfigure}[b]{0.32\textwidth}
  \centering
  \includegraphics[width=\linewidth]{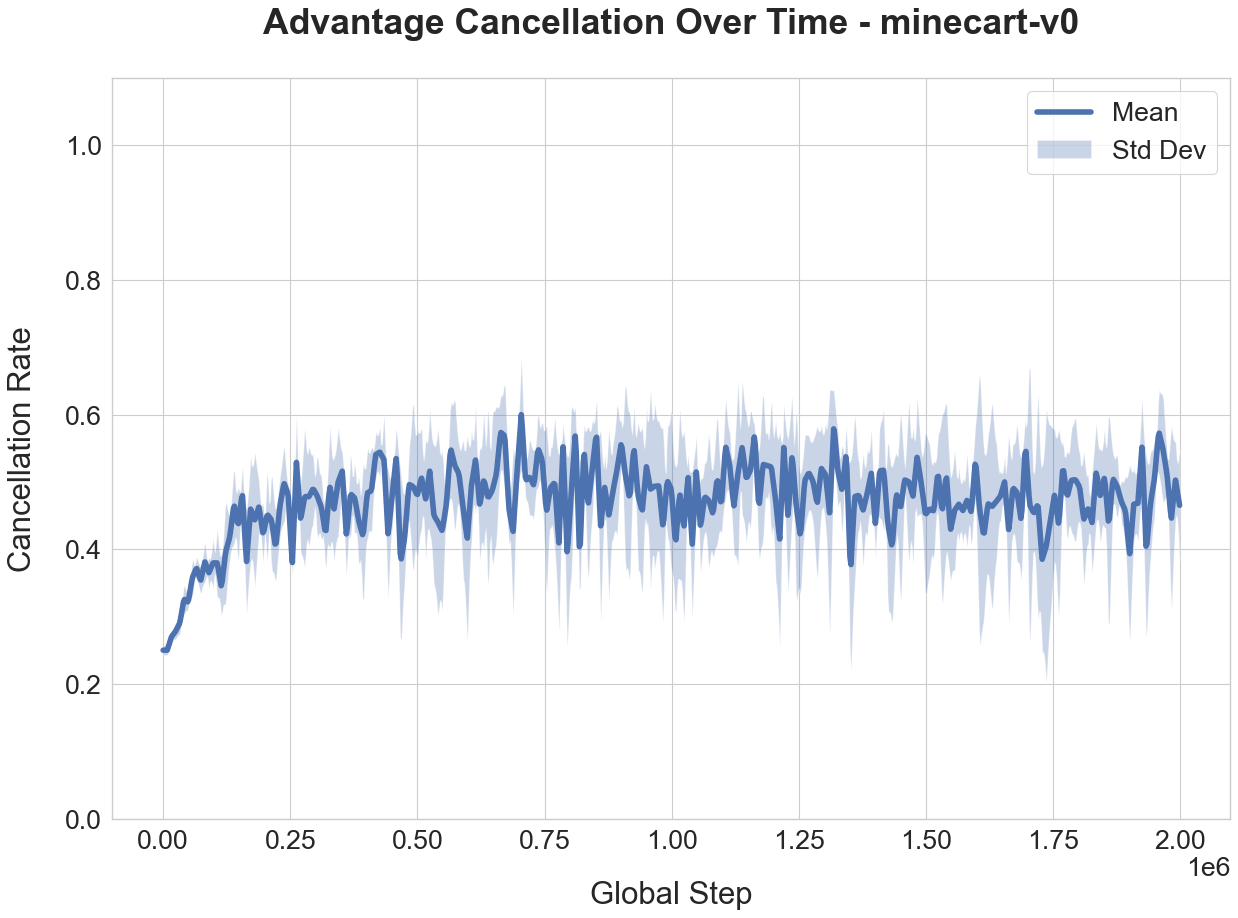}
  \caption{Minecart-3d}
\end{subfigure}\hfill
\begin{subfigure}[b]{0.32\textwidth}
  \centering
  \includegraphics[width=\linewidth]{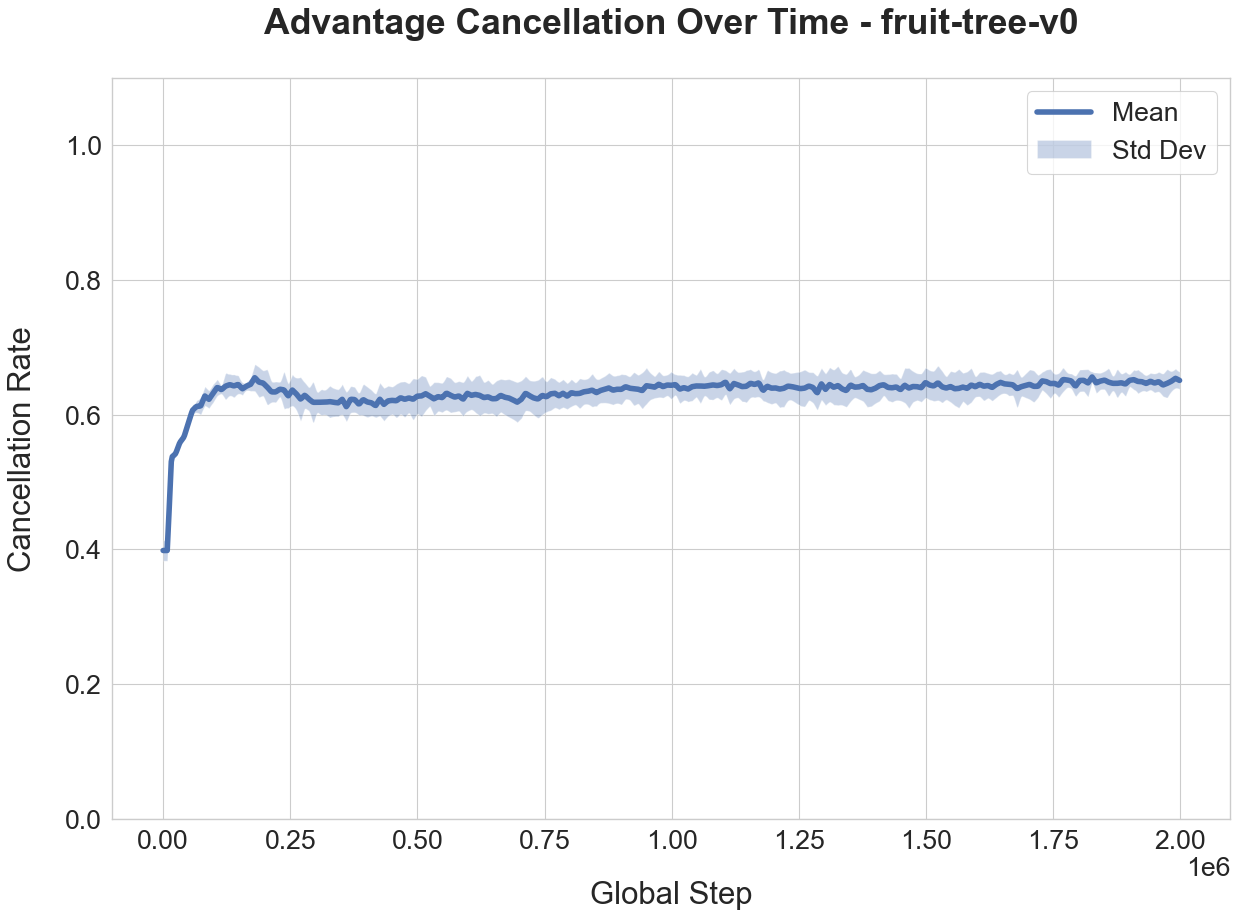}
  \caption{Fruit Tree}
\end{subfigure}

\caption{\textbf{Empirical Advantage Cancellation Rates under Early Scalarization (ES).} Across diverse environments, applying preference weights before trust-region stabilization results in substantial, sustained reduction in advantage magnitude. In dense continuous tasks (Ant-v5) and highly discontinuous fronts (Fruit Tree), ES can reduce up to 50--65\% of the available advantage magnitude at every update, significantly degrading the effective learning signal.}
\label{fig:advantage_cancellation_grid}
\end{figure*}

\begin{figure*}[!h]
\centering
\begin{subfigure}[b]{0.32\textwidth}
  \centering
  \includegraphics[width=\linewidth]{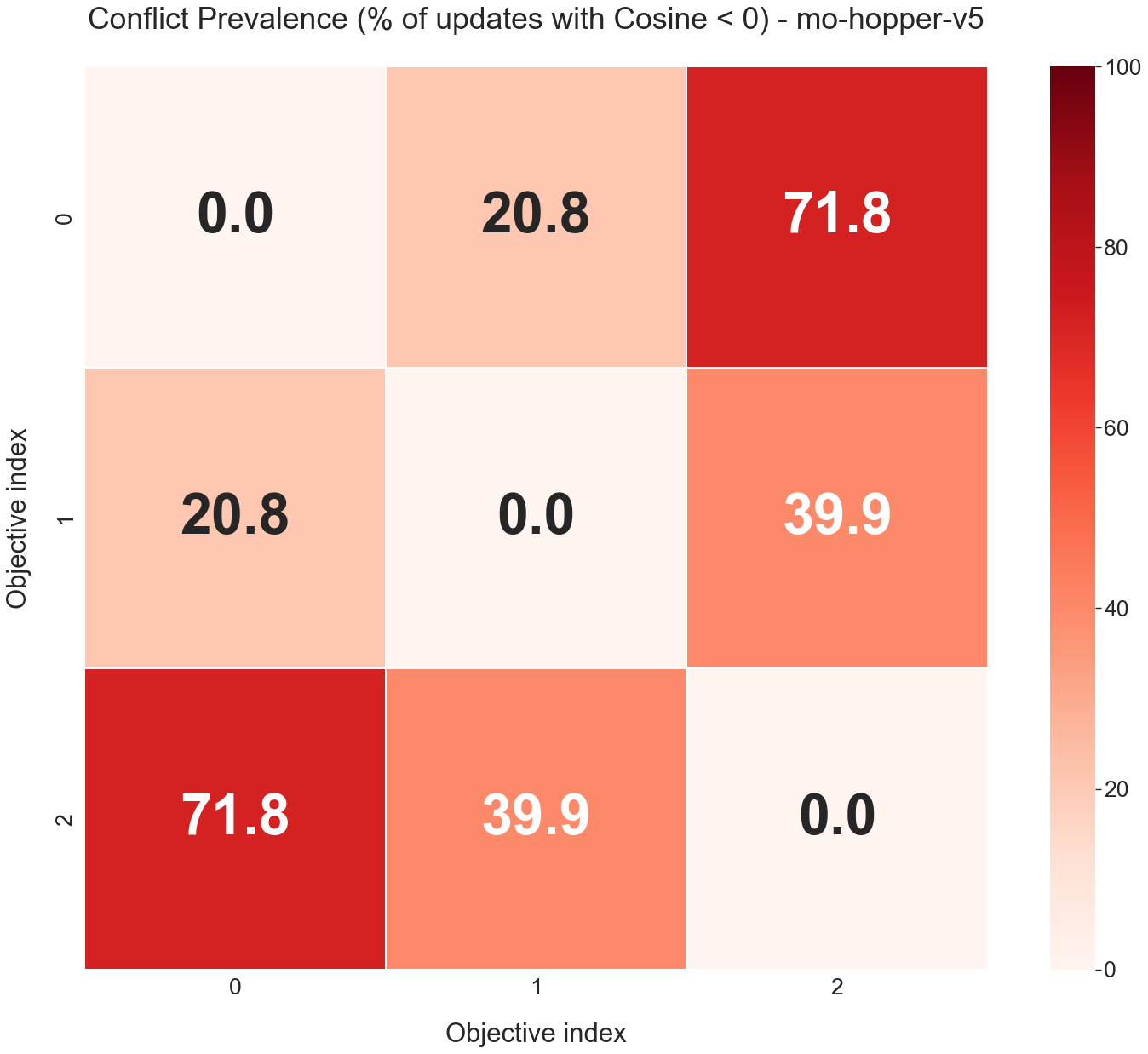}
  \caption{Hopper-3d}
\end{subfigure}\hfill
\begin{subfigure}[b]{0.32\textwidth}
  \centering
  \includegraphics[width=\linewidth]{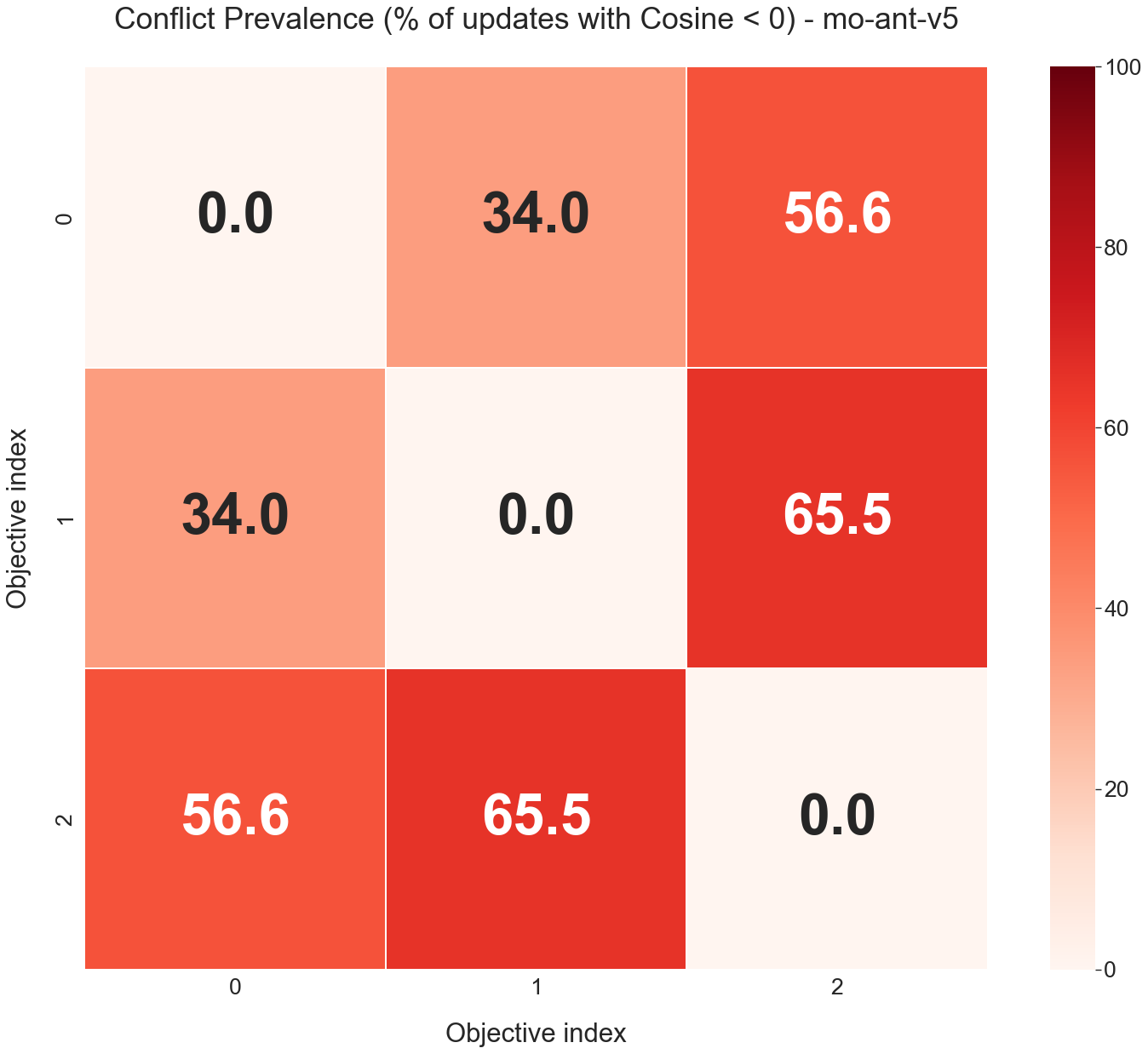}
  \caption{Ant-2d}
\end{subfigure}\hfill
\begin{subfigure}[b]{0.32\textwidth}
  \centering
  \includegraphics[width=\linewidth]{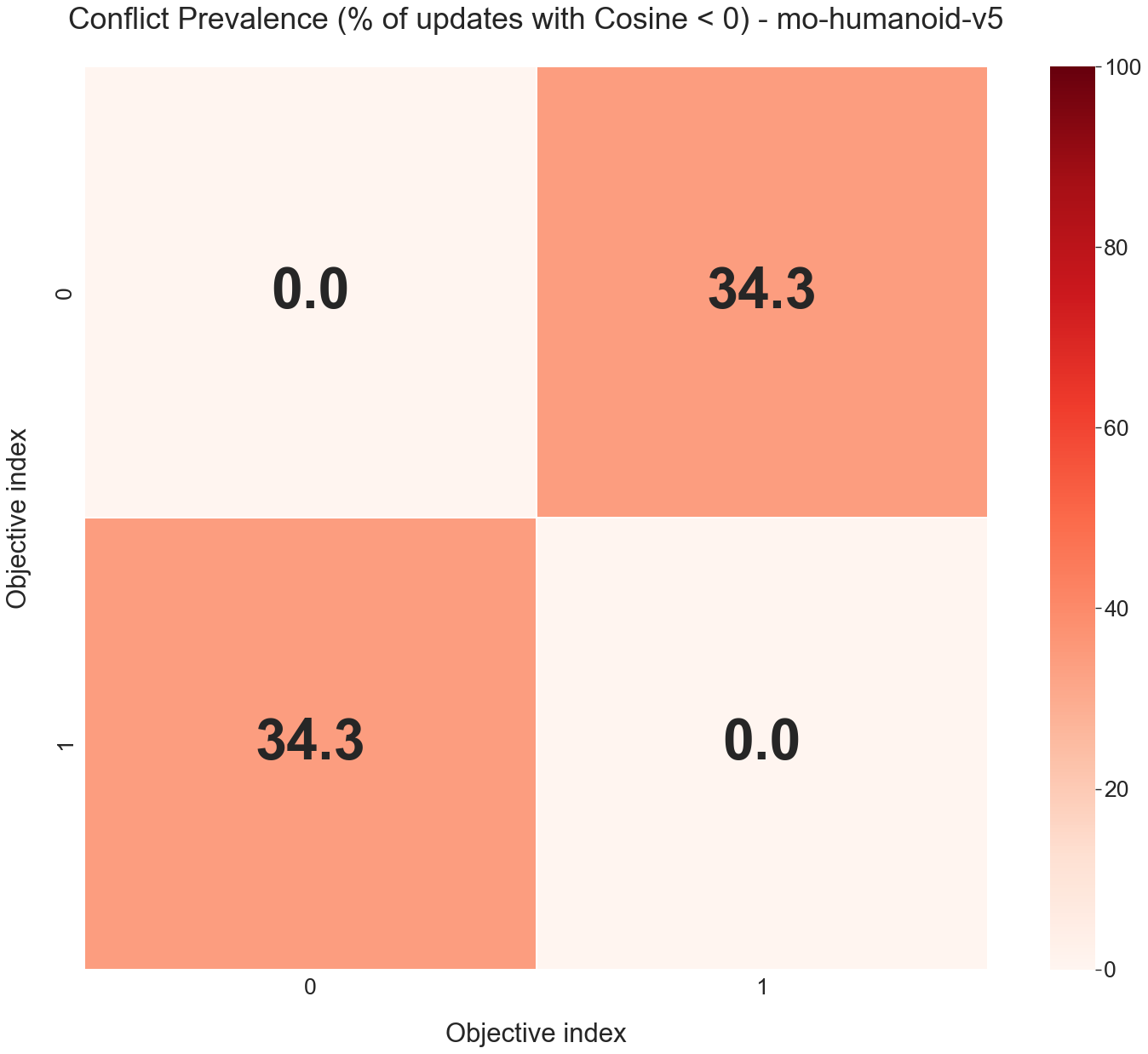}
  \caption{Humanoid-2d}
\end{subfigure}

\vspace{0.3cm}

\begin{subfigure}[b]{0.32\textwidth}
  \centering
  \includegraphics[width=\linewidth]{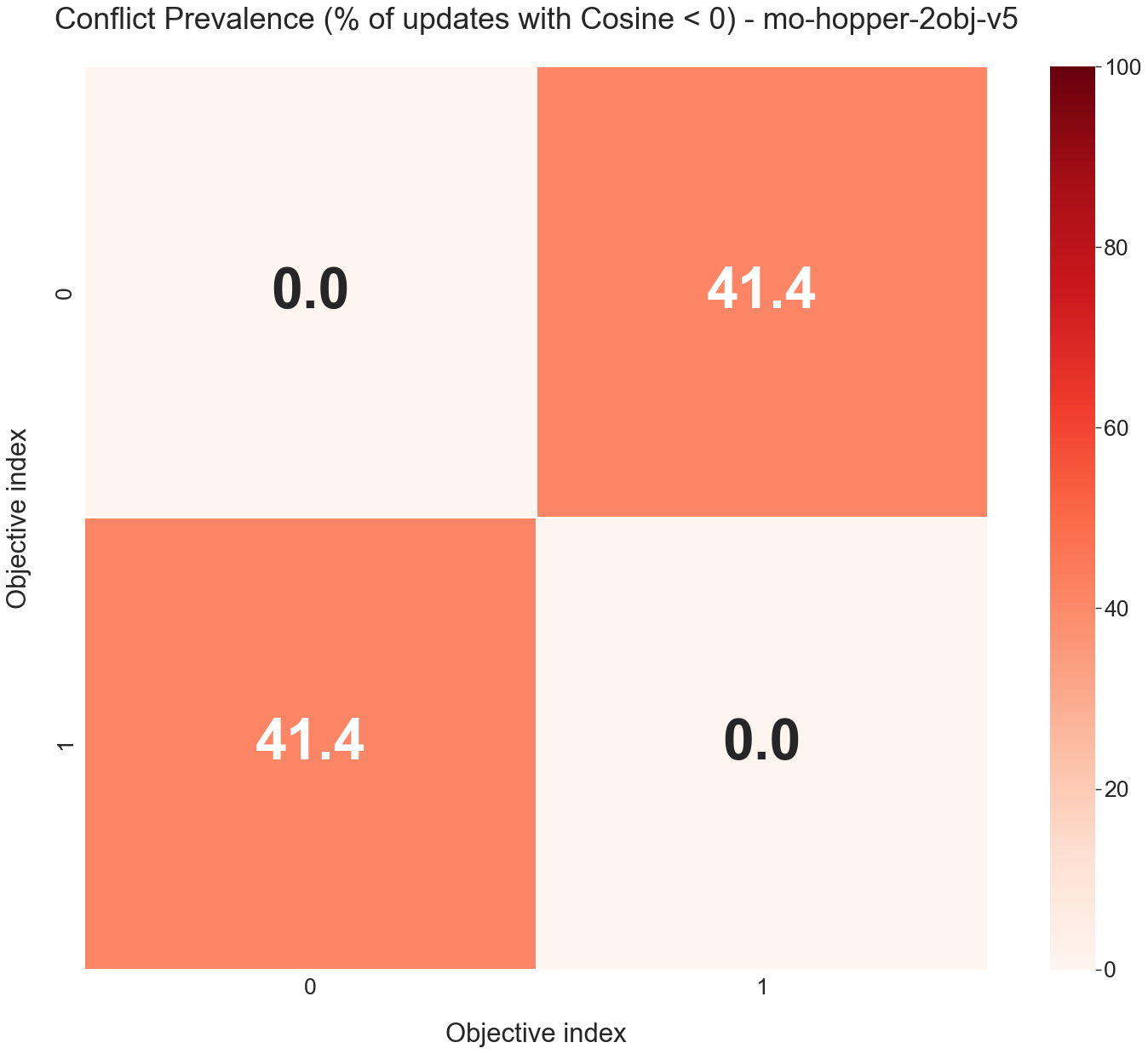}
  \caption{Hopper-2d}
\end{subfigure}\hfill
\begin{subfigure}[b]{0.32\textwidth}
  \centering
  \includegraphics[width=\linewidth]{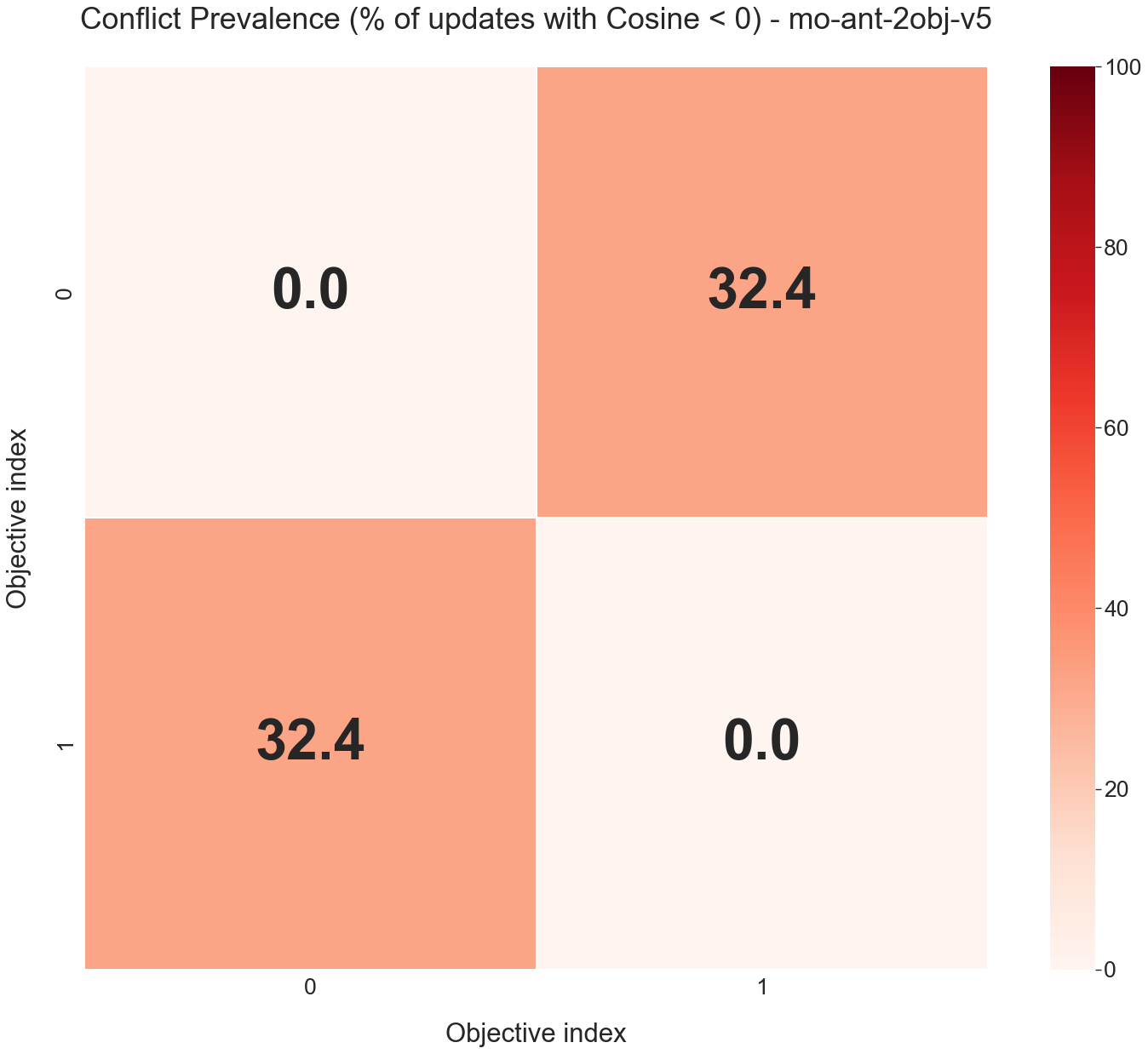}
  \caption{Ant-2d}
\end{subfigure}\hfill
\begin{subfigure}[b]{0.32\textwidth}
  \centering
  \includegraphics[width=\linewidth]{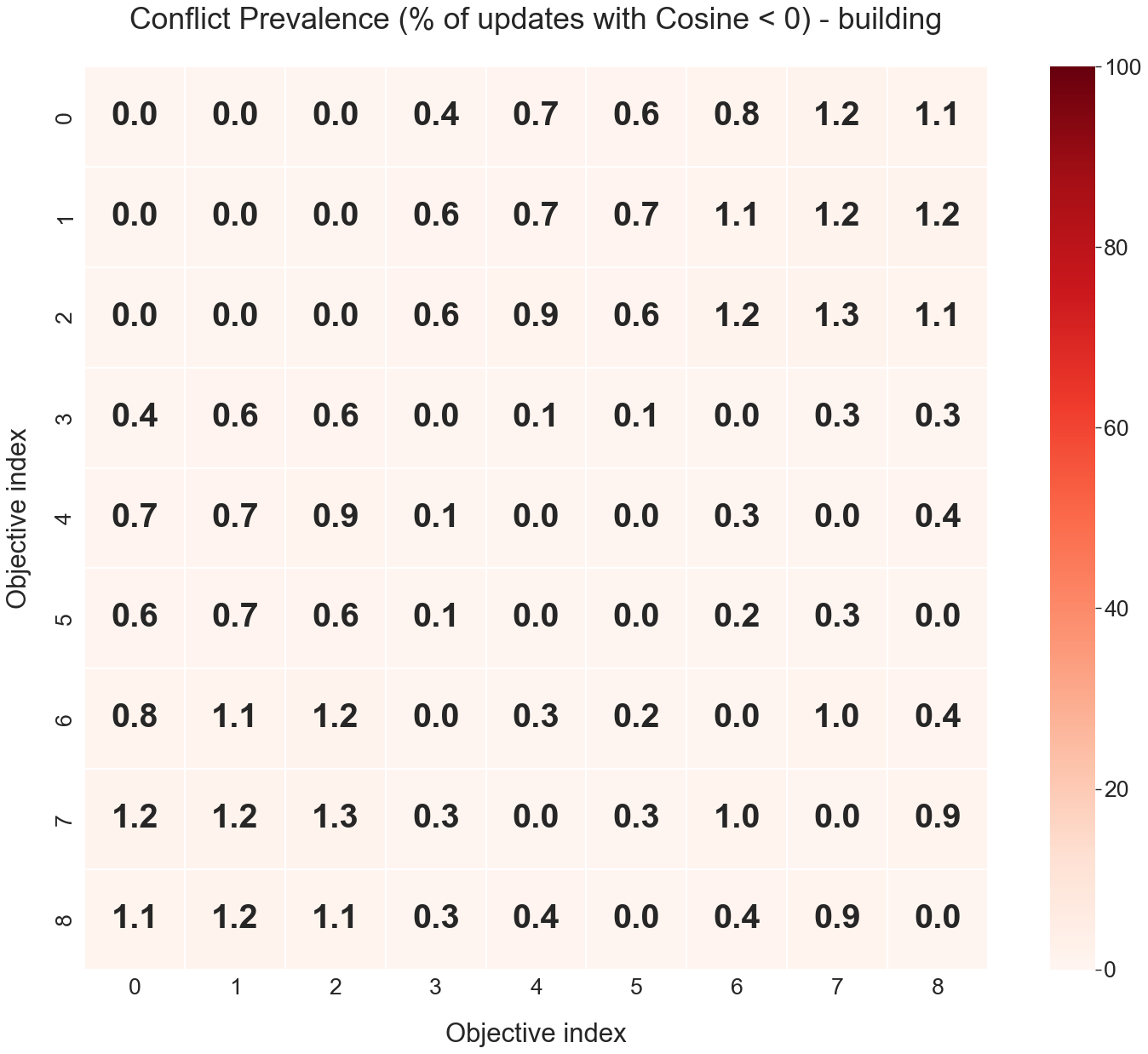}
  \caption{Building-9d}
\end{subfigure}

\vspace{0.3cm}

\begin{subfigure}[b]{0.32\textwidth}
  \centering
  \includegraphics[width=\linewidth]{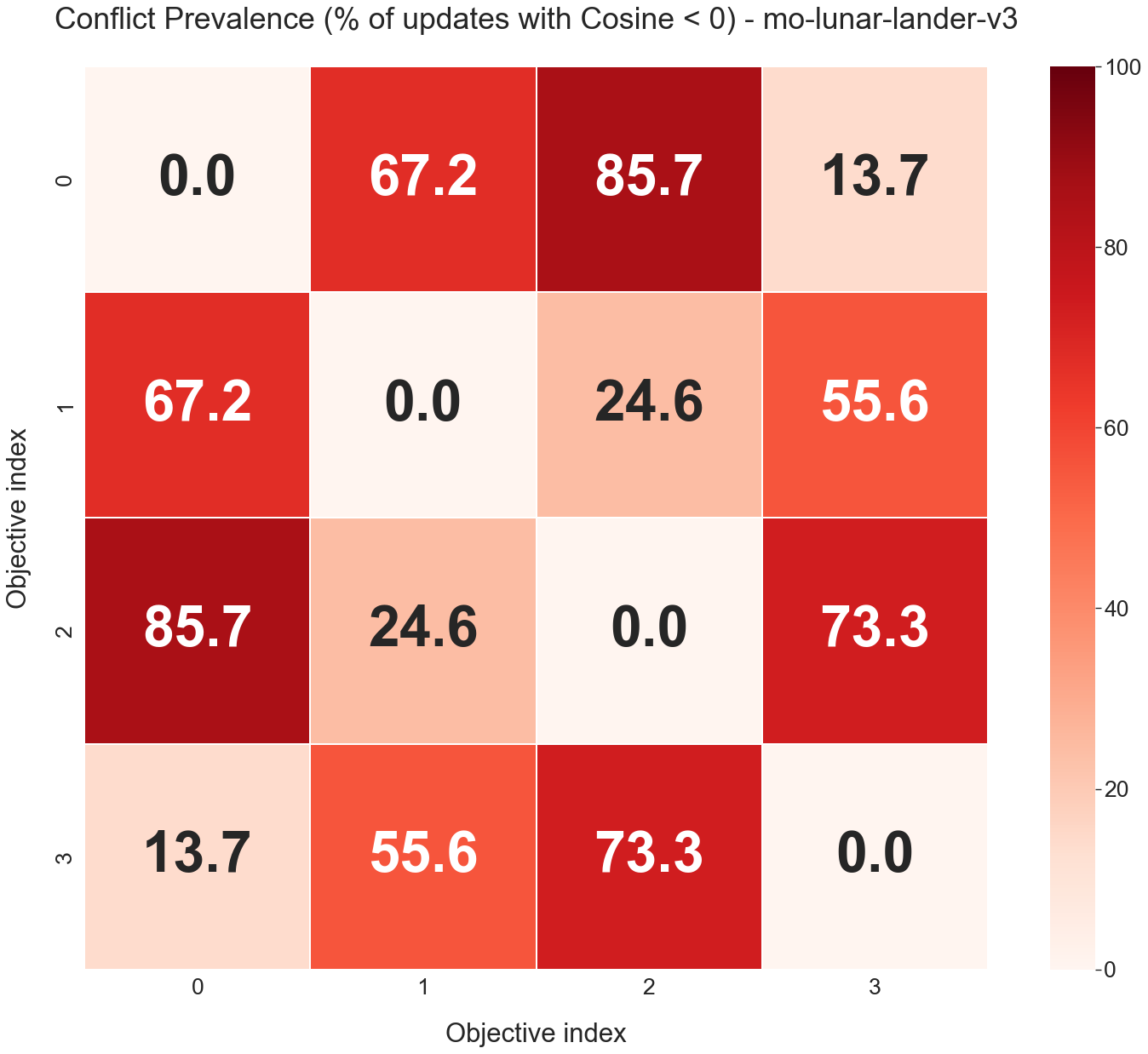}
  \caption{Lunar-Lander-4d}
\end{subfigure}\hfill
\begin{subfigure}[b]{0.32\textwidth}
  \centering
  \includegraphics[width=\linewidth]{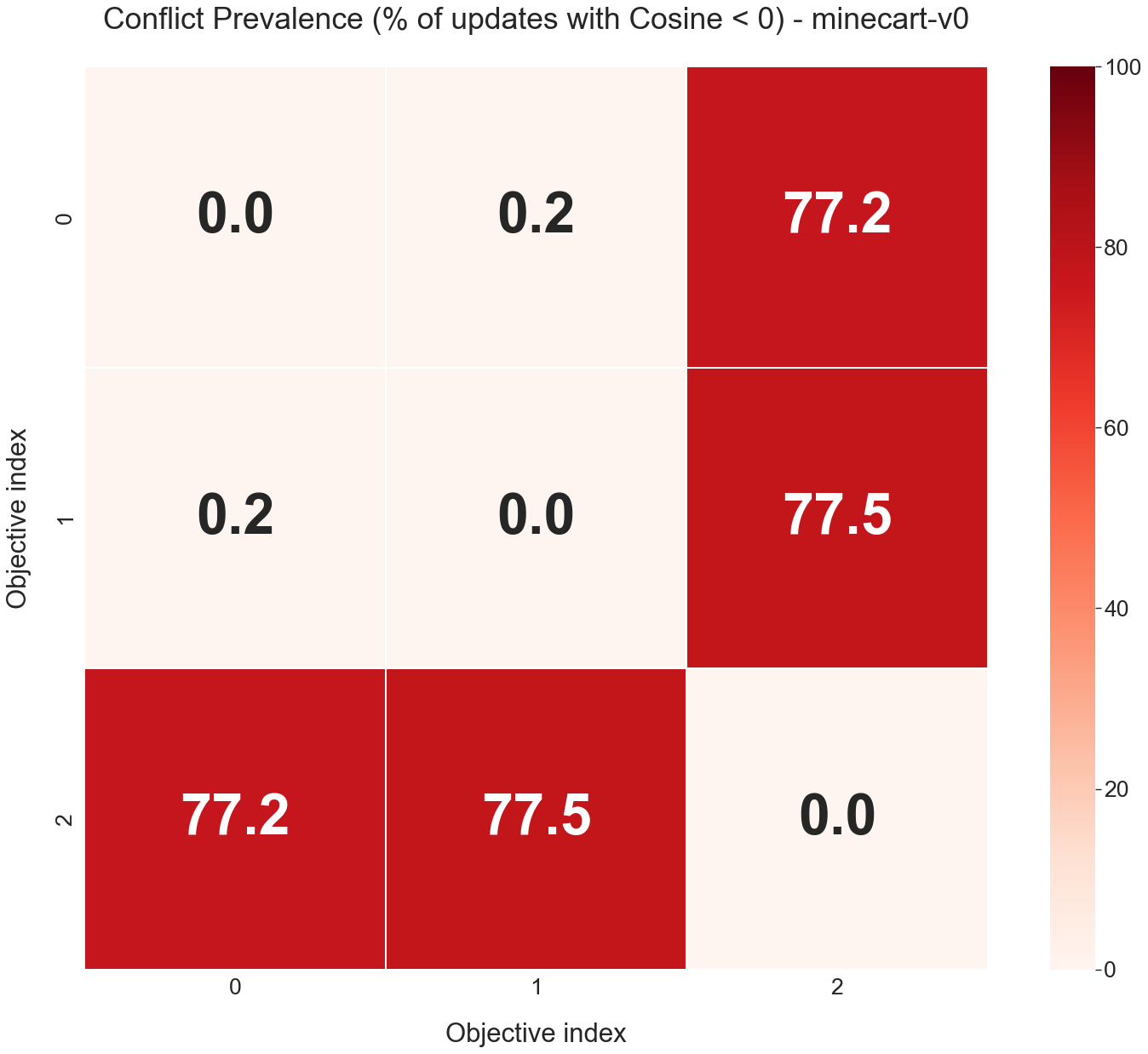}
  \caption{Minecart-3d}
\end{subfigure}\hfill
\begin{subfigure}[b]{0.32\textwidth}
  \centering
  \includegraphics[width=\linewidth]{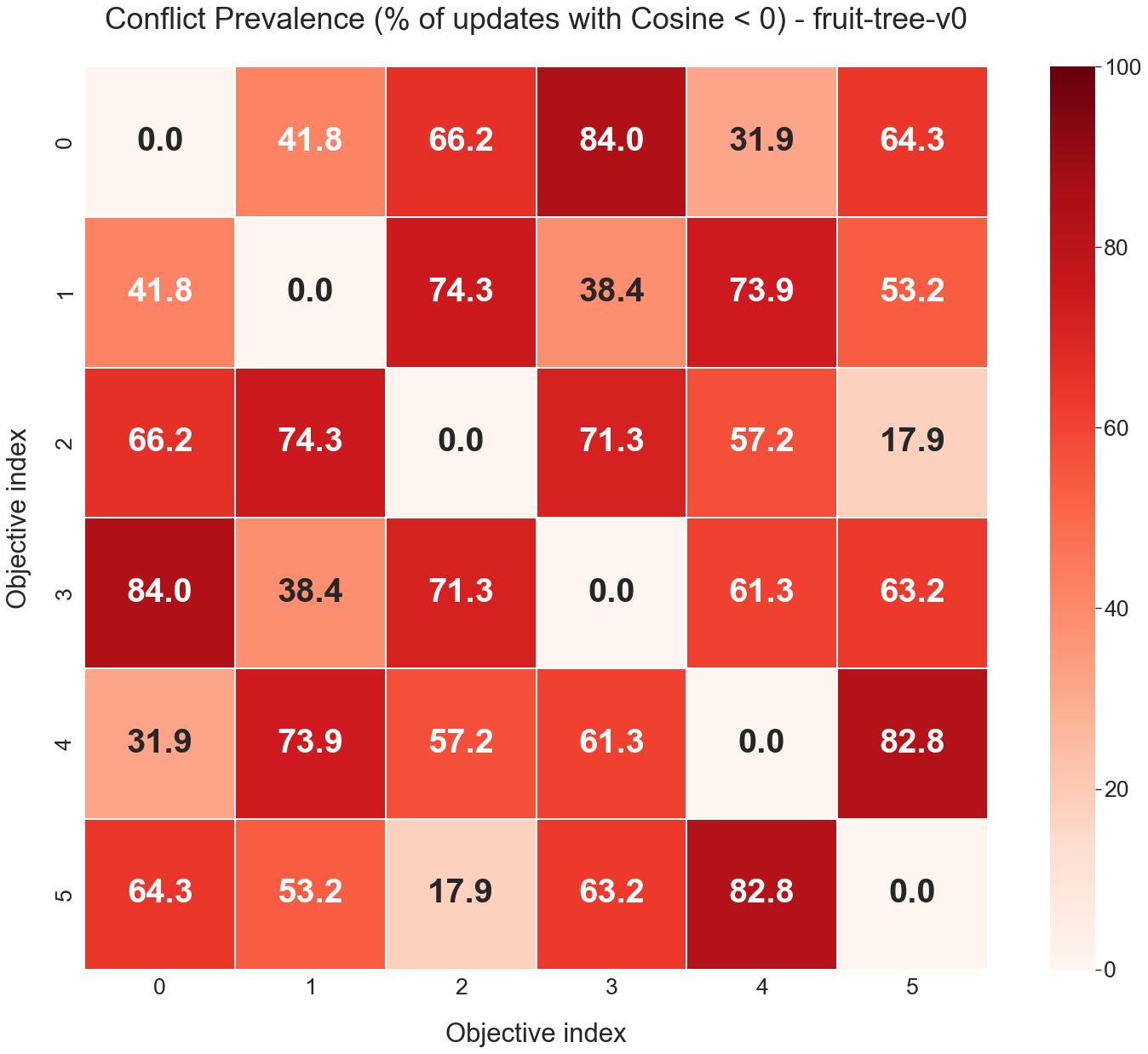}
  \caption{FruitTree}
\end{subfigure}

\caption{\textbf{Prevalence of Gradient Conflict.} This figure illustrates the frequency of gradient opposition between pairs of objectives throughout training. As policies mature and approach the Pareto front, the prevalence of conflicting gradients dramatically increases, often stabilizing between 30\% and 55\%. Under Early Scalarization, this high frequency of conflict leads to frequent and sustained signal cancellation under Early Scalarization.}
\label{fig:gradient_prevalence_grid}
\end{figure*}

%% file: documentBody/F-theory.tex
\section{Theoretical Analysis of Multi-Objective PPO Formulations}
\label{sec:LSW_theory}

To justify the design of Late-Stage Weighting (LSW), we provide a formal
comparative analysis of three methods for integrating preference weights into
PPO. The central question is not merely \emph{when} weights are applied, but
whether PPO's clipping operator receives a single scalarized signal or
independently bounded per-objective signals. We show that this ordering of
operations changes the gradient that reaches the policy whenever the trust
region binds: in the unclipped regime all three variants are equivalent, but
under conflicting objectives and active clipping, ES can zero out a learning
signal that LSW preserves.

\subsection{Formal Definitions of MORL-PPO Variants}

Let
\[
\rho_t(\theta) \;=\; \frac{\pi_\theta(a_t \mid s_t, \omega)}
                          {\pi_{\theta_{\text{old}}}(a_t \mid s_t, \omega)}
\]
be the importance sampling ratio and $\mathbf{A}_t = [A_t^{(1)}, \dots, A_t^{(d)}]$
the vector of per-objective advantages. We compare three ways to incorporate
$\omega \in \Delta^{d-1}$ into a PPO-style surrogate, distinguished by
\emph{where} a scalar first enters the clipping operator.

\begin{enumerate}

\item \textbf{Early Scalarization (ES).}
Advantages are combined into a single scalar \emph{before} any per-objective
processing; this scalar enters the PPO surrogate directly:
\begin{equation}
    \mathcal{L}_{\text{clip}}^{\text{ES}}(\theta)
    = \mathbb{E}_t\!\Big[
        \min\!\Big(
            \rho_t(\theta)\,(\omega^\top\mathbf{A}_t),\;
            \operatorname{clip}(\rho_t(\theta),1{-}\epsilon,1{+}\epsilon)
            \,(\omega^\top\mathbf{A}_t)
        \Big)
    \Big].
    \label{eq:es}
\end{equation}

\item \textbf{Mid-stage Vectorial Scalarization (MVS).}
Per-objective advantages are individually normalized, then weighted and summed
into a single scalar, which is passed to one PPO surrogate:
\begin{equation}
    \mathcal{L}_{\text{actor}}^{\text{MVS}}(\theta)
    = -\,\mathbb{E}_t\!\Big[
        \min\!\Big(
            \rho_t(\theta)\,\textstyle\sum_i \omega_i \hat{A}_t^{(i)},\;
            \operatorname{clip}(\rho_t(\theta),1{-}\epsilon,1{+}\epsilon)
            \,\textstyle\sum_i \omega_i \hat{A}_t^{(i)}
        \Big)
    \Big],
    \label{eq:mvs}
\end{equation}
where $\hat{A}_t^{(i)}$ denotes the normalized advantage for objective $i$.

\item \textbf{Late-Stage Weighting (LSW).}
The PPO clipped surrogate is computed \emph{independently} for each objective
on its raw advantage; preference weights are applied only to the resulting
per-objective surrogate losses:
\begin{equation}
    \mathcal{L}_{\text{actor}}^{\text{LSW}}(\theta)
    = -\sum_{i=1}^{d} \omega_i\;\mathbb{E}_t\!\Big[
        \min\!\Big(
            \rho_t(\theta)\,A_t^{(i)},\;
            \operatorname{clip}(\rho_t(\theta),1{-}\epsilon,1{+}\epsilon)
            \,A_t^{(i)}
        \Big)
    \Big].
    \label{eq:lsw}
\end{equation}

\end{enumerate}

In ES and MVS the clipping operator is applied to a single pre-aggregated
scalar ($\omega^\top\mathbf{A}_t$ and $\sum_i \omega_i \hat{A}_t^{(i)}$
respectively). In LSW the clipping operator is applied to each $A_t^{(i)}$
independently, and preferences are combined only afterwards. The theoretical
analysis below characterizes exactly when and how these three formulations
produce different policy gradients.

\subsection{Gradient-Level Mechanism}
\label{sec:gradient_mechanism}

We first establish that ES and LSW are indistinguishable whenever the trust
region is inactive, and that their divergence arises solely from how the
clipping operator selects its active branch under conflicting objective signs.

\begin{lemma}[Unclipped equivalence, clipped divergence]
\label{lem:clipping_mechanism}
Fix a state--action pair $(s_t, a_t)$, a ratio $\rho_t(\theta)$, a clipping
parameter $\epsilon > 0$, and a preference vector $\omega \in \Delta^{d-1}$.
Denote the per-sample integrands of ES and LSW by
$\ell^{\text{ES}}_t$ and $\ell^{\text{LSW}}_t$ respectively (dropping the
expectation for pointwise analysis).

\textbf{(i) Unclipped regime.} If $\rho_t(\theta) \in [1-\epsilon, 1+\epsilon]$,
then for every configuration of $\mathbf{A}_t$,
\[
\nabla_\theta \ell^{\text{ES}}_t
\;=\;
\nabla_\theta \ell^{\text{LSW}}_t
\;=\;
\nabla_\theta \rho_t(\theta)\,\cdot\,\sum_{i=1}^d \omega_i A_t^{(i)}.
\]

\textbf{(ii) Clipped regime, aligned signs.} If all nonzero $A_t^{(i)}$ share
the same sign, ES and LSW select the same branch of the $\min(\cdot,\cdot)$
operator, and their gradients again coincide up to the identical scalar factor
$\sum_i \omega_i A_t^{(i)}$.

\textbf{(iii) Clipped regime, conflicting signs.} Suppose there exist
$i, j$ with $A_t^{(i)} > 0$, $A_t^{(j)} < 0$, and $\omega_i, \omega_j > 0$,
and suppose $\rho_t(\theta) \notin [1-\epsilon, 1+\epsilon]$. Then ES and LSW
in general select different branches of the $\min$ operator and produce
different policy gradients:
\[
\nabla_\theta \ell^{\text{ES}}_t \;\neq\; \nabla_\theta \ell^{\text{LSW}}_t.
\]
Specifically, partition the objectives into
$\mathcal{I}^+ = \{i : A_t^{(i)} > 0\}$ and $\mathcal{I}^- = \{i : A_t^{(i)} < 0\}$,
and suppose $\rho_t(\theta) > 1+\epsilon$. Then:
\begin{align*}
\nabla_\theta \ell^{\text{LSW}}_t
&\;=\; \nabla_\theta \rho_t(\theta)\,\cdot\,
      \sum_{i \in \mathcal{I}^-} \omega_i A_t^{(i)}, \\
\nabla_\theta \ell^{\text{ES}}_t
&\;=\;
\begin{cases}
\mathbf{0}
& \text{if } \omega^\top \mathbf{A}_t \geq 0, \\
\nabla_\theta \rho_t(\theta)\,\cdot\,\omega^\top \mathbf{A}_t
& \text{if } \omega^\top \mathbf{A}_t < 0.
\end{cases}
\end{align*}
A symmetric statement holds for $\rho_t(\theta) < 1-\epsilon$, with the roles
of $\mathcal{I}^+$ and $\mathcal{I}^-$ exchanged.
\end{lemma}

\begin{proof}
In the unclipped regime, $\operatorname{clip}(\rho_t, 1-\epsilon, 1+\epsilon) = \rho_t$,
so both arguments of the $\min$ operator coincide in both ES and LSW. The per-sample
integrands reduce to $\rho_t \cdot \omega^\top \mathbf{A}_t$ and
$\sum_i \omega_i \rho_t A_t^{(i)}$, which are algebraically identical. Taking
gradients yields part (i).

In the clipped regime, the $\min$ operator selects the clipped branch
(gradient zero with respect to $\theta$) when that branch is smaller, and the
unclipped branch (gradient $\nabla\rho_t \cdot A$) otherwise. For a scalar
advantage $A$, a standard case analysis yields that the clipped branch is
selected iff $A > 0$ and $\rho_t > 1+\epsilon$, or $A < 0$ and $\rho_t < 1-\epsilon$.
Under ES this decision is made once based on $\omega^\top \mathbf{A}_t$; under
LSW it is made independently for each $A_t^{(i)}$.

Under aligned signs (part ii), every $A_t^{(i)}$ has the same sign as
$\omega^\top \mathbf{A}_t$, so LSW's per-objective decisions all match ES's
single decision, and both surrogates contribute the same gradient.

Under conflicting signs and $\rho_t > 1+\epsilon$ (part iii): each $i \in \mathcal{I}^+$
satisfies the clipping condition, so LSW's $i$-th term contributes zero gradient;
each $i \in \mathcal{I}^-$ fails the clipping condition, so LSW's $i$-th term
contributes $\omega_i \nabla \rho_t \cdot A_t^{(i)}$. Summing yields the stated
LSW gradient. For ES, the single clipping decision is governed by the sign of
$\omega^\top \mathbf{A}_t$: if nonnegative, the clipped branch is selected
and the gradient vanishes; if negative, the unclipped branch contributes
$\nabla \rho_t \cdot \omega^\top \mathbf{A}_t$. \qed
\end{proof}

\begin{example}[ES annihilates a valid learning signal]
\label{ex:worked}
Let $d=2$, $\omega = (\tfrac{1}{2}, \tfrac{1}{2})$,
$A_t^{(1)} = +1$, $A_t^{(2)} = -0.9$, so $\omega^\top \mathbf{A}_t = +0.05$.
Suppose $\rho_t(\theta) > 1+\epsilon$.

\textbf{ES:} The scalar advantage $+0.05$ is positive and $\rho_t$ is above the
upper clip, so the $\min$ selects the clipped branch. \emph{$\nabla \ell^{\text{ES}}_t = 0$.}

\textbf{LSW:}
\begin{itemize}[leftmargin=2em, topsep=0pt]
    \item $i=1$: $A_t^{(1)} > 0$ and $\rho_t > 1+\epsilon \Rightarrow$ clipped branch, zero gradient.
    \item $i=2$: $A_t^{(2)} < 0$ and $\rho_t > 1+\epsilon \Rightarrow$ unclipped branch.
\end{itemize}
Total gradient: $\nabla \ell^{\text{LSW}}_t
= \omega_2 \cdot A_t^{(2)} \cdot \nabla \rho_t = -0.45 \cdot \nabla \rho_t \neq 0$.

ES produces no learning signal despite the presence of a large, valid
per-objective advantage on objective 2; LSW retains the signal from the
conflicting objective and correctly pushes the policy away from actions that
hurt it. A symmetric example with $\rho_t < 1-\epsilon$ shows ES discarding
gradient from the positive objective.
\end{example}

Lemma~\ref{lem:clipping_mechanism} isolates the mechanism: the three
formulations are equivalent until the trust region binds, at which point
ES's single clipping decision is made on an aggregated signal that can
destructively combine conflicting objectives, while LSW's per-objective
decisions each reflect the sign and magnitude of the corresponding objective.
The consequence is that ES can entirely suppress a gradient that LSW preserves.

\subsection{Signal Magnitude as a Supporting Bound}

The gradient-level analysis above is the primary justification for LSW.
For completeness, we note a coarser magnitude bound that characterizes the
pre-clipping signal available to each formulation.

\begin{proposition}[Magnitude bound on the pre-clip scalar]
\label{prop:magnitude_bound}
Let $A_t^\omega := \omega^\top \mathbf{A}_t$ denote the scalarized advantage
under ES or MVS, and let
$M_{\text{LSW}} := \sum_{i=1}^d \omega_i |A_t^{(i)}|$
denote the sum of weighted per-objective magnitudes available to LSW prior
to clipping. Then
\[
|A_t^\omega| \;\leq\; M_{\text{LSW}},
\]
with strict inequality whenever there exist $i, j$ with
$A_t^{(i)} A_t^{(j)} < 0$ and $\omega_i, \omega_j > 0$.
\end{proposition}

\begin{proof}
Triangle inequality: $|\omega^\top \mathbf{A}_t|
= |\sum_i \omega_i A_t^{(i)}|
\leq \sum_i \omega_i |A_t^{(i)}| = M_{\text{LSW}}$, with strictness iff at
least two terms have opposite signs. \qed
\end{proof}

Proposition~\ref{prop:magnitude_bound} quantifies a necessary consequence of
ES and MVS: whenever objectives conflict, the scalar entering the clipping
operator has strictly smaller magnitude than the sum of its components. This
bound is a coarser statement than Lemma~\ref{lem:clipping_mechanism}, which
shows that the gradient itself, not merely its pre-clip magnitude, can
collapse to zero under ES.

\subsection{Structural Hierarchy}

\begin{proposition}[Algebraic equivalence of ES and MVS under homogeneous conditions]
\label{prop:es_mvs_equiv}
In the absence of per-objective preprocessing (i.e., $\hat{A}_t^{(i)} = A_t^{(i)}$
for all $i$), ES and MVS collapse to the same scalar $\omega^\top \mathbf{A}_t$
before clipping and are therefore algebraically identical:
\[
\mathcal{L}_{\text{clip}}^{\text{ES}}(\theta)
\;=\;
\mathcal{L}_{\text{actor}}^{\text{MVS}}(\theta).
\]
\end{proposition}

\begin{proof}
Without preprocessing, the MVS scalar is $\sum_i \omega_i A_t^{(i)} = \omega^\top \mathbf{A}_t$,
which is identical to the ES scalar. Both surrogates therefore operate on
the same collapsed signal, and their clipping decisions coincide. \qed
\end{proof}

\begin{proposition}[Preprocessing distortion in MVS]
\label{prop:mvs_distortion}
Under per-objective advantage normalization $\mathcal{N}(\cdot)$, which is
not positively homogeneous of degree~1, MVS computes
$\mathcal{N}(\omega_i A_t^{(i)})$ while LSW-style normalization computes
$\omega_i \mathcal{N}(A_t^{(i)})$. These quantities differ whenever
preferences are non-uniform, and MVS therefore distorts the per-objective
scale established by $\mathcal{N}$ relative to LSW.
\end{proposition}

\begin{proof}
Standard normalization (e.g., subtract mean, divide by standard deviation)
is not positively homogeneous of degree~1: for $c \neq 1$,
$\mathcal{N}(c x) \neq c \, \mathcal{N}(x)$ in general, because the scaling
of $x$ by $c$ changes both the mean and standard deviation of the batch.
Applying $\omega_i$ before $\mathcal{N}$ (as MVS does) therefore distorts
the normalized magnitude of objective $i$ relative to the LSW ordering,
which applies $\mathcal{N}$ first and $\omega_i$ afterwards. \qed
\end{proof}

\begin{corollary}[Robustness hierarchy]
\label{cor:hierarchy}
Combining Lemma~\ref{lem:clipping_mechanism}, Proposition~\ref{prop:magnitude_bound},
Proposition~\ref{prop:es_mvs_equiv}, and Proposition~\ref{prop:mvs_distortion}:
\[
\text{LSW} \;\succ\; \text{MVS} \;\succeq\; \text{ES},
\]
in terms of gradient preservation under conflicting objectives and an active
trust region. LSW strictly dominates ES and MVS whenever objective signs
conflict and $\rho_t$ is outside the clipping window. MVS and ES coincide in
the absence of per-objective preprocessing and MVS strictly exceeds ES only
when such preprocessing is present, via the scale distortion of
Proposition~\ref{prop:mvs_distortion}.
\end{corollary}

\subsection{Empirical Validation of the Hierarchy}

Table~\ref{tab:ablation_mvs} directly validates the theoretical hierarchy by
isolating the contribution of each scalarization stage. ES collapses
per-objective signals at the clipping stage, MVS partially recovers structure
through per-objective normalization before the single-scalar collapse, and LSW
avoids the collapse entirely by decoupling the clipping decisions.

\begin{table}[h]
\centering
\caption{Ablation of scalarization strategy across three environments,
validating the hierarchy $\text{LSW} \succ \text{MVS} \succeq \text{ES}$.
The ES$\to$MVS gain reflects the secondary benefit of per-objective
normalization before collapse; the MVS$\to$LSW gain reflects the primary
benefit of per-objective clipping, consistent with
Lemma~\ref{lem:clipping_mechanism}. $0^*$ denotes collapse to a
single-point front. Results averaged over 5 seeds for $1.5\times 10^6$ steps.}
\label{tab:ablation_mvs}
\begin{tabular}{llccc}
\toprule
Environment & Metric & ES & MVS & LSW \\
\midrule
\multirow{3}{*}{Humanoid-2d}
  & HV ($10^5\uparrow$) & $1.50 \pm 0.17$ & $2.67 \pm 0.11$ & $\mathbf{3.76 \pm 0.11}$ \\
  & EU ($10^2\uparrow$) & $2.87 \pm 0.22$ & $4.01 \pm 0.24$ & $\mathbf{5.11 \pm 0.09}$ \\
  & SP ($10^4\downarrow$) & $0^*$           & $0.004 \pm 0.004$ & $\mathbf{0.003 \pm 0.001}$ \\
\midrule
\multirow{3}{*}{Hopper-2d}
  & HV ($10^5\uparrow$) & $1.23 \pm 0.03$ & $1.25 \pm 0.16$ & $\mathbf{1.30 \pm 0.03}$ \\
  & EU ($10^2\uparrow$) & $2.38 \pm 0.05$ & $2.45 \pm 0.25$ & $\mathbf{2.47 \pm 0.01}$ \\
  & SP ($10^2\downarrow$) & $0.08 \pm 0.02$ & $0.87 \pm 1.04$ & $\mathbf{0.26 \pm 0.31}$ \\
\midrule
\multirow{3}{*}{Ant-2d}
  & HV ($10^5\uparrow$) & $1.53 \pm 0.11$ & $1.60 \pm 0.19$ & $\mathbf{1.91 \pm 0.18}$ \\
  & EU ($10^2\uparrow$) & $2.71 \pm 0.13$ & $2.75 \pm 0.22$ & $\mathbf{3.14 \pm 0.21}$ \\
  & SP ($10^3\downarrow$) & $0.18 \pm 0.07$ & $0.28 \pm 0.10$ & $0.66 \pm 0.40$ \\
\bottomrule
\end{tabular}
\end{table}

Three patterns are noteworthy. First, the HV and EU ordering ES $<$ MVS $<$ LSW
holds consistently across all three environments, directly confirming
Corollary~\ref{cor:hierarchy}. Second, the MVS$\to$LSW gain is substantially
larger than the ES$\to$MVS gain, particularly in Humanoid-2d — consistent with
per-objective clipping being the primary mechanism
(Lemma~\ref{lem:clipping_mechanism}) and per-objective normalization a
secondary one (Proposition~\ref{prop:mvs_distortion}). Third, ES collapses
entirely to a single-point front in Humanoid-2d (SP $= 0^*$), while MVS
partially recovers diversity and LSW achieves robust front coverage,
illustrating that the structural failure of single-scalar clipping is most
damaging in high-dimensional environments where objective conflict is most
severe.

\subsection{Implications}

\begin{itemize}

\item \textbf{The mechanism is at the clipping stage, not the aggregation stage.}
Lemma~\ref{lem:clipping_mechanism} shows that ES and LSW produce identical
gradients whenever the trust region is inactive. The divergence arises
entirely from how the clipping operator resolves its branch selection under
conflicting objective signs. Changing the order of weighting and clipping is
therefore not a cosmetic reordering: it changes which gradients survive the
trust region.

\item \textbf{ES can zero out a valid learning signal.}
Example~\ref{ex:worked} exhibits a case where a large per-objective advantage
exists and is consistent with policy improvement, yet ES produces zero
gradient because the aggregated scalar $\omega^\top \mathbf{A}_t$ happens to
fall on the clipped side. LSW retains this signal. This failure mode is
structural; it cannot be fixed by tuning $\epsilon$ or by reward
normalization.

\item \textbf{ES and MVS share the same clipping-stage limitation.}
Proposition~\ref{prop:es_mvs_equiv} establishes that without preprocessing
they are algebraically identical. Proposition~\ref{prop:mvs_distortion}
establishes that with preprocessing, MVS improves the pre-clip scale but
still passes a single scalar to the clipping operator, so the mechanism of
Lemma~\ref{lem:clipping_mechanism} still applies. The empirical gap between
ES and MVS is modest relative to the MVS-to-LSW gap
(Table~\ref{tab:ablation_mvs}), consistent with this analysis.

\item \textbf{LSW preserves PPO's intended semantics per-objective.}
By computing the clipped surrogate independently for each objective before
any preference weighting, LSW ensures the trust region bounds each
objective's contribution individually — exactly the behavior PPO's clip is
calibrated for. Preference weights then combine already-stabilized
contributions, and the resulting gradient reflects a preference-weighted
combination of correctly-bounded per-objective updates.

\end{itemize}

%% file: documentBody/G-regularizer.tex
\section{Theoretical Analysis of the Scaled Diversity Regularizer}
\label{sec:diversity_theory}

In this section, we provide a formal argument that the scaled diversity regularizer explicitly penalizes representational mode collapse by establishing a target separation in policy space proportional to separation in preference space.

\begin{definition}[Representational Mode Collapse]
A preference-conditioned policy $\pi_\theta(a|s,\omega)$ exhibits \textbf{mode collapse} if there exists a region in the preference simplex of non-zero measure where two distinct preference vectors, $\omega_A \neq \omega_B$, produce statistically indistinguishable action distributions for all states. Formally, for some $\delta = \|\omega_A-\omega_B\|_1 > 0$,
\[
    D_{KL}(\pi_{\theta}(\cdot|s, \omega_A) \,\|\, \pi_{\theta}(\cdot|s, \omega_B)) = 0 \quad \text{for all } s \in \mathcal{S}.
\]
\end{definition}

\begin{proposition}[Separation Induced by Diversity Regularizer]
\label{lem:diversity_theory}
Let the scaled diversity regularizer be defined as:
\[
    \mathcal{L}_{\text{diversity}}(\theta) = \mathbb{E}_{s \sim d^\pi, \omega, \omega'} \Big[ \big( D_{KL}(\pi_{\theta}(\cdot|s,\omega) \,\|\, \pi_{\theta}(\cdot|s,\omega')) - \alpha \|\omega-\omega'\|_1 \big)^2 \Big].
\]
A global minimum of this regularization term is achieved when:
\[
    D_{KL}(\pi_{\theta}(\cdot|s,\omega_A)\,\|\,\pi_{\theta}(\cdot|s,\omega_B)) = \alpha \|\omega_A-\omega_B\|_1
\quad \text{almost everywhere under } s \sim d^\pi.
\]
Consequently, any policy exhibiting mode collapse ($\exists \omega_A \neq \omega_B$ where $D_{KL} = 0$) incurs a strictly positive penalty proportional to $\alpha^2\|\omega_A-\omega_B\|_1^2$, introducing a gradient signal that discourages collapsed representations during joint optimization.
\end{proposition}

\begin{proof}
Because $\mathcal{L}_{\text{diversity}}(\theta)$ is an expectation of a non-negative squared error, $\mathcal{L}_{\text{diversity}}(\theta) \ge 0$. The global minimum $\mathcal{L}_{\text{diversity}}(\theta) = 0$ is achieved when the squared term is zero almost everywhere under the sampling distribution, requiring $D_{KL}(\pi_{\theta}(\cdot|s,\omega_A) \,\|\, \pi_{\theta}(\cdot|s,\omega_B)) = \alpha \|\omega_A-\omega_B\|_1$ for sampled states and preference pairs.

If a policy exhibits mode collapse, meaning $D_{KL} = 0$ for some $\delta = \|\omega_A-\omega_B\|_1 > 0$, the squared error for that pair evaluates to $(0 - \alpha\delta)^2 = \alpha^2\delta^2 > 0$. Therefore, mode collapse strictly sub-optimizes the regularization term, establishing a continuous penalty against representational collapse.
\end{proof}

\paragraph{Convexity and Objective Dominance.}
While Proposition~\ref{lem:diversity_theory} shows that the regularizer encourages preference-proportional separation, in practice, $\mathcal{L}_{\text{diversity}}$ is optimized jointly with the PPO surrogate $\mathcal{L}_{\text{policy}}$. This balance is crucial: the regularizer penalizes insufficient separation only when distinct, high-performing behaviors are feasible. If the underlying environment admits only a finite set of Pareto-optimal solutions (e.g., highly discontinuous discrete environments), the primary RL objective will dominate the gradients, causing the policy to converge to these true environmental solutions even if the resulting KL divergence falls short of the target $\alpha\delta$. Thus, the diversity term smoothly \emph{encourages} distinct solutions along the Pareto front without forcing the network to hallucinate suboptimal policies in regions where the environment affords none.

%% file: documentBody/I-environments.tex
\section{Environment Descriptions}

\paragraph{Minecart.} A multi-objective task where an agent controls a cart in a 2D continuous environment. The state space is 70dimensional. The agent selects from a discrete action space (6 actions) to navigate the environment and mine for resources. The reward is a 3-dimensional vector, with conflicting objectives for collecting two different types of ore while minimizing fuel consumption. The agent must learn to navigate between different mining locations, creating a trade-off between the types of ore collected and the fuel expended. The hypervolume reference point is $[-1, -1, -200]$ and the $\gamma$ used to calculate the returns to construct the front is 0.99 

\paragraph{Lunar-Lander-4D.} 
A multi-objective version of the classic Lunar Lander control problem.
The state space is 8-dimensional ($\mathcal{S} \subseteq \mathbb{R}^8$), containing the lander's position, velocity, angle, and leg contact information.
The agent selects from a 4-dimensional discrete action space ($\mathcal{A}$) representing firing the main engine, the left or right orientation thrusters, or doing nothing.
The reward is a 4-dimensional vector, with separate components for the landing outcome (success or crash), a distance-based shaping reward, main engine fuel cost, and side engine fuel cost. 
The hypervolume reference point is $[-101, -1001, -101, -101]$ and the $\gamma$ used to calculate the returns to construct the front is 0.99

\paragraph{Hopper-2D.}
A continuous-control task based on the Hopper-v5 environment, where a one-legged robot must learn a trade-off between forward movement and jumping height.
The observation space is 11-dimensional ($\mathcal{S} \subseteq \mathbb{R}^{11}$), capturing joint angles and velocities, while the 3-dimensional continuous action space ($\mathcal{A} \subseteq \mathbb{R}^3$) controls joint torques.
The two objectives are the agent's forward velocity and its vertical displacement, both augmented with a small control cost. The hypervolume reference point is $[-100, -100]$ and the $\gamma$ used to calculate the returns to construct the front is 0.99.

\paragraph{Hopper-3D.}
An extension of MO-Hopper-2D with an explicit third objective: minimizing control cost.
The agent must now learn a three-way trade-off between forward velocity, jumping height, and energy efficiency, which is defined as the negative squared magnitude of the action vector ($-\sum a_i^2$).
The observation space remains 11-dimensional and the action space 3-dimensional. The hypervolume reference point is $[-100, -100, -100]$ and the $\gamma$ used to calculate the returns to construct the front is 0.99.

\paragraph{Ant-2D.}
Based on the Ant-v5 robot, this continuous-control task involves a quadruped navigating a 2D plane.
The state space is 105-dimensional ($\mathcal{S} \subseteq \mathbb{R}^{105}$), representing joint positions, velocities, and contact forces.
The action space is 8-dimensional ($\mathcal{A} \subseteq \mathbb{R}^8$), controlling the torques at each leg joint.
The 2-dimensional reward vector consists of the agent's x-velocity ($v_x$) and y-velocity ($v_y$).
The hypervolume reference point is $[-100, -100]$ and the $\gamma$ used to calculate the returns to construct the front is 0.99.

\paragraph{Ant-3D.}
An extension of MO-Ant-2D with an additional objective for control cost.
The agent must optimize its x-velocity and y-velocity while simultaneously minimizing the magnitude of applied joint torques ($-2\sum a_i^2$).
The state space remains 105-dimensional and the action space 8-dimensional, but the objective space is now 3-dimensional. The hypervolume reference point is $[-100, -100, -100]$ and the $\gamma$ used to calculate the returns to construct the front is 0.99.

\paragraph{Humanoid-2D.}
Based on the Humanoid-v5 robot, this environment features one of the most complex state spaces in common benchmarks, with 348 state dimensions ($\mathcal{S} \subseteq \mathbb{R}^{348}$) and a 17-dimensional continuous action space ($\mathcal{A} \subseteq \mathbb{R}^{17}$).
The task presents two highly conflicting objectives: maximizing forward velocity ($v_x$) and minimizing energy consumed, represented by a control cost penalty ($-10\sum a_i^2$). The hypervolume reference point is $[-100, -100]$ and the $\gamma$ used to calculate the returns to construct the front is 0.99.

\paragraph{Building-9D.}
A complex thermal control task for a large commercial building, featuring a 29-dimensional state space ($\mathcal{S} \subseteq \mathbb{R}^{29}$) and a 23-dimensional continuous action space ($\mathcal{A} \subseteq \mathbb{R}^{23}$).
The agent must manage the heating supply across 23 zones.
The three core objectives (minimizing energy cost, temperature deviation, and power ramping) are calculated independently for each of the building's three floors, resulting in a challenging, high-dimensional 9-objective problem. The hypervolume reference point is $[0, 0, 0, 0, 0, 0, 0, 0, 0]$ and the $\gamma$ used to calculate the returns to construct the front is 1.

%% file: documentBody/J-expDetails.tex
\section{Experimental Details}
\label{sec:experiment_details}
The PPO specific hyperparameters are the following:
\begin{itemize}
    \item Number of environments: 4
    \item Learning Rate: 0.0003
    \item Batch Size: 512
    \item Number of minibatches: 32
    \item Gamma: 0.995
    \item GAE lambda: 0.95
    \item Surrogate Clip Threshold: 0.2
    \item Entropy Loss coefficient: 0
    \item Value function loss coefficient: 0.5
    \item Normalize Advantages, Normalize Observations, Normalize rewards: True
    \item Max gradient Norm: 0.5
\end{itemize}

For the actor network, we initialized the final layer with logstd value of 0.
For humanoid and ant benchmarks, the logstd value was -1.
We performed every experiment with 5 random seeds to find confidence intervals.
In all cases, both actor and critic networks had 2 hidden layers with 64 neurons in each layer.
The activations were tanh, with the final layer having no activation.
For all experiments, the action diversity loss parameter $\lambda$ was 0.01 and $\alpha = 1$

We trained all baselines and \toolname{} on a M3 Ultra 96gb RAM Mac Studio. 

All baselines used the same number of environment interactions, network architecture size, and PPO parameters. We have specifically ensured that all baselines and \toolname{} use a fixed number of steps, despite parallelism. We would like to note that C-MORL and PG-MORL spawn multiple workers (6-8), each using the initialization + improvement budget, which causes their total interaction count to get multiplied by the number of processes. In the supplementary material, we have corrected the code to ensure these baselines use the correct and same budget as all others. We provide full logs and code to ensure reproducibility

\subsection{Memory Comparison}

To demonstrate the substantial memory advantage of \toolname{} over the state-of-the-art C-MORL algorithm, we compare the total number of parameters required to represent all policies along the Pareto front. Because C-MORL is a multi-policy approach, it trains a separate actor--critic pair for each preference, meaning that every point on the front corresponds to an independent network $\pi_{\text{cmorl}}$ that maps only the state to an action. In contrast, \toolname{} learns a single preference-conditioned policy $\pi_{\text{d3po}}(a \mid s, \omega)$ capable of representing the entire continuum of optimal trade-offs with one unified actor--critic model.

Table~\ref{tab:memory} reports the parameter counts and corresponding float32 memory footprint. Notably, C-MORL imposes a practical cap of 200 policies per environment due to memory and training limitations, whereas \toolname{} can represent an unbounded number of solutions because preference variation is handled through conditioning rather than training separate networks. In fact, for the Building-9D environment, we observed more than $2000$ distinct Pareto-optimal preference vectors, all represented seamlessly by a single \toolname{} model.

\begin{table}[!h]\centering

\caption{Parameter counts and storage for \toolname{} and C-MORL.}
\label{tab:memory}
\begin{tabular}{lcc}
\toprule
Environment & \toolname{} (params, MB) & C-MORL (params, MB) \\
\midrule
Ant-2D & 23,314 (0.089\,MB) & 3,770,852 (14.385\,MB) \\
Ant-3D & 23,507 (0.090\,MB) & 735,776 (2.807\,MB) \\
Hopper-2D & 10,632 (0.041\,MB) & 1,361,052 (5.192\,MB) \\
Hopper-3D & 10,825 (0.041\,MB) & 2,062,200 (7.867\,MB) \\
Humanoid-2D & 55,588 (0.212\,MB) & 1,326,408 (5.060\,MB) \\
Building-9D & 16,887 (0.064\,MB) & 3,043,000 (11.608\,MB) \\
\bottomrule
\end{tabular}
\end{table}

%% file: documentBody/K-Reproducibility.tex
\section*{Reproducibility Statement}

We have taken several steps to ensure the reproducibility of our work. 
All algorithmic details of \toolname{} are fully specified in Section~\ref{sec:method}, with pseudocode provided in Algorithm~\ref{alg:mo-ppo-single}. 
Our theoretical results are supported by complete proofs in Appendix~\ref{sec:LSW_theory} \ref{sec:diversity_theory}, where all assumptions are stated explicitly. 
The experimental setup, including environment details, hyperparameters, and evaluation metrics, is documented in Section~\ref{sec:experiments} and further expanded in Appendix~\ref{sec:experiment_details}. 
We use publicly available benchmark environments without modification, and we describe our training protocols and data processing steps in detail. 
Anonymous source code implementing \toolname{}, along with scripts for reproducing all experiments and figures, is included in the supplementary material. 
Together, these resources ensure that both the theoretical and empirical contributions of this paper are fully reproducible.


\section{Limitations and Future Work}
\label{sec:limitations}

While \toolname{} establishes a new state-of-the-art across our evaluated benchmarks, its reliance on linear scalarization inherently bounds its discoverable solutions to the Convex Coverage Set (CCS). In environments featuring strictly concave Pareto regions (e.g., the classic Deep Sea Treasure), \toolname{}---like all linearly scalarized baselines---will mathematically bridge the concavity and recover only the convex extremes. Extending our decomposed advantage architecture to target concave interiors is a primary direction for future research. This could be achieved by integrating \toolname{} with non-linear utility functions (e.g., Tchebycheff scalarization) or archive-based evolutionary strategies, which bypass linear constraints by explicitly anchoring exploration to previously discovered Pareto-optimal points.

Furthermore, our architectural innovations are currently tailored to the on-policy setting. Because Late-Stage Weighting (LSW) explicitly relies on the independent stabilization of per-objective signals via the PPO clipping operator, it does not trivially transfer to off-policy algorithms (e.g., SAC or TD3), which lack direct advantage clipping and instead rely on Q-value bootstrapping. Formulating an off-policy equivalent of LSW remains an open challenge. 

D$^3$PO inherits the convergence properties of PPO: the actor objective
(Eq.~\ref{eq:lsw} combined with the scaled diversity regularizer) is
smooth in $\theta$ under standard policy parameterizations, and thus
standard stochastic approximation results for smooth non-convex objectives
apply under the usual assumptions (bounded-variance gradient estimates,
Robbins--Monro step sizes, bounded iterates). We do not claim a
D$^3$PO-specific convergence result: the diversity regularizer is a smooth
penalty that does not alter the regularity of the actor objective, and
Late-Stage Weighting is a finite sum of clipped per-objective surrogates
that inherits the smoothness structure of standard PPO. Convergence
analysis for PPO itself remains an active area, and the same caveats
apply here.

Finally, regarding objective dimensionality, \texttt{Building-9d} represents the highest-objective continuous control benchmark currently available. While we have not empirically evaluated \toolname{} beyond nine objectives, the mathematical formulation of LSW and the scaled diversity regularizer scales strictly linearly with the number of objectives ($O(d)$). Therefore, we anticipate that the framework will gracefully scale to massively multi-objective environments without the exponential computational explosion common to hypervolume-maximization methods.

\section{Explorations of Nonlinear Scalarization in Multi-Objective PPO:
Theoretical Motivations and Empirical Failure Modes}

A central limitation of linear scalarization is well-established: it can only recover
the convex regions of the Pareto frontier. This motivates the investigation of nonlinear
scalarization strategies that could, in principle, recover interior and concave regions.
We systematically explored several such approaches within the PPO framework. Our
central finding is that while nonlinear scalarization is theoretically appealing, its
introduction into PPO consistently violates the implicit assumptions of the clipping
operator and degrades performance. This section presents a unified account of these
explorations and a precise diagnosis of why each approach fails.

\paragraph{The Trust Region Assumption and Why It Matters.}
PPO's clipped surrogate is derived under a specific set of assumptions: the advantage
function is taken as a fixed, linear signal, and the clipping operator enforces a local
trust region that bounds the policy ratio to $[1-\epsilon, 1+\epsilon]$. Crucially,
the calibration of this trust region — the choice of $\epsilon$ and the implicit step
size it induces — assumes that the gradient of the surrogate loss is proportional to
the advantage. Any transformation applied to the advantage or the surrogate loss that
introduces curvature breaks this proportionality. The clipping threshold that was
appropriate for a linear advantage is no longer appropriate for a nonlinearly
transformed one: the trust region is simultaneously too large in directions where the
nonlinearity amplifies the gradient and too small in directions where it suppresses it.
This misalignment is the root cause of instability across all nonlinear variants we
explored.

\paragraph{Log-Sum-Exp and Smooth Maximum Scalarization.}
The most theoretically motivated nonlinear approach replaces the linear combination
of per-objective surrogate losses with a log-sum-exp transformation, approximating
a smooth maximum over objectives:
\[
\mathcal{L}^{\text{LSE}}(\theta) = \frac{1}{\tau} \log \sum_{i=1}^d
\exp\!\left(\tau \, \omega_i \, \mathcal{L}^{(i)}_{\text{clip}}(\theta)\right).
\]
This corresponds to optimizing a risk-sensitive utility that adaptively emphasizes
whichever objective is currently performing worst, and in principle can recover concave
regions of the Pareto frontier by concentrating gradient pressure where improvement
is most needed.

In practice, this formulation violates the trust region in a systematic way. The
gradient of $\mathcal{L}^{\text{LSE}}$ with respect to $\theta$ involves softmax
weights $s_i \propto \exp(\tau \omega_i \mathcal{L}^{(i)}_{\text{clip}})$ applied
to the per-objective gradients. These weights are data-dependent and change across
minibatches, introducing a second source of curvature that is invisible to the clipping
operator. When $\tau$ is large, the softmax sharpens and concentrates gradient mass on
a single objective, producing updates far outside the trust region intended by the
clipping threshold. When $\tau$ is small, the transformation approaches linearity and
the method recovers LSW, but without the nonlinear coverage benefit that motivated
it. We found no stable operating regime across environments: small $\tau$ wastes the
nonlinearity, large $\tau$ destabilizes training. The fundamental issue is that the
clipping operator was not designed to control updates whose effective magnitude depends
on the loss landscape rather than the policy ratio alone.

\paragraph{Anchor-Based Advantage Shaping.}
A second approach augments the per-objective advantage with a gap term derived from
a Pareto-optimal reference point maintained in an archive:
\[
\tilde{A}^{(i)}_t = A^{(i)}_t + \beta \left(G^{(i)}_{\text{target}} -
V^{(i)}_\phi(s_t, \omega)\right),
\]
where $G^{(i)}_{\text{target}}$ is the $i$-th return of the nearest Pareto-dominant
solution. The intuition is that the augmented advantage provides a global signal
pulling the policy toward underexplored regions of the frontier rather than relying
solely on local temporal difference errors.

This approach fails for a fundamental reason: the augmented quantity is no longer a
valid advantage in the sense required by PPO. The advantage function in PPO must be
action-dependent — it represents the relative merit of a specific action over the
baseline — and the trust region is calibrated under the assumption that the advantage
reflects local action quality. The gap term $G^{(i)}_{\text{target}} -
V^{(i)}_\phi(s_t, \omega)$ is state-dependent but not action-dependent; it is a
constant with respect to $a_t$ at any given state. Injecting it into the advantage
therefore biases the policy ratio in a direction that is decoupled from action
selection, producing gradients that push the policy to regions of high predicted
return regardless of which actions actually achieve them. Empirically this manifests
as unstable updates and convergence to high-return states via unreliable action
sequences. Replacing the value estimate with empirical returns to restore
action-dependence introduces high variance that the clipping operator cannot
adequately control.

\paragraph{Nash Bargaining-Inspired Adaptive Weighting.}
Motivated by cooperative game theory, Nash scalarization defines weights inversely
proportional to the margin between current returns and a nadir point:
\[
\omega^{\text{Nash}}_i \propto \frac{1}{V^{(i)}_\phi(s_t, \omega) - n_i + \delta},
\]
where $n_i$ is the nadir value for objective $i$ and $\delta > 0$ is a stability
constant. These adaptive weights correspond to the gradient of a log-transformed
product utility, a concave objective whose maximizer lies on the Pareto frontier and
is not confined to the convex hull. The method therefore has genuine theoretical
potential to recover non-convex trade-offs.

The practical failure is twofold. First, the adaptive weights are data-dependent and
change across training steps and minibatches, reintroducing the same trust region
miscalibration as log-sum-exp: the clipping threshold $\epsilon$ was chosen under
fixed weights, and varying weights change the effective step size in a way the
clipping operator cannot compensate. Second, the inverse relationship between weights
and margins amplifies noise in value estimates. When $V^{(i)}_\phi$ is poorly
estimated — which is inevitable early in training and in high-dimensional environments
— small errors in the denominator produce large, unstable weights. Clamping the
weights mitigates this but destroys the Nash property. Replacing trajectory-dependent
estimates with a learned global value function stabilizes the weights but removes the
adaptivity that gives Nash scalarization its appeal. We found no implementation that
preserved both stability and the theoretical coverage benefit.

\paragraph{Hybrid Selective Nonlinearity.}
The observation that PPO's clipping operator is itself a piecewise nonlinearity
suggests a natural hybrid: apply nonlinear aggregation only when the surrogate is in
its linear (unclipped) regime, where additional curvature is needed, and fall back to
linear aggregation when clipping is already providing nonlinear pressure. While this
aligns conceptually with PPO's structure, it introduces a discrete partition of
objectives into clipped and unclipped sets that changes across timesteps and
minibatches. The gradient of the combined objective becomes discontinuous at the
boundaries of this partition, and the effective weighting of objectives changes
abruptly as training progresses. These discontinuities create instability that
outweighs the benefit of selective nonlinearity in all environments tested.

\paragraph{Unified Diagnosis.}
Across all variants, the failure mode is the same: PPO's trust region is a precisely
calibrated mechanism designed for a specific class of surrogate objectives, and any
transformation that introduces data-dependent curvature into the advantage or loss
renders the clipping threshold inappropriate. The trust region becomes simultaneously
too permissive in directions where the transformation amplifies gradients and too
restrictive where it suppresses them, producing either instability or premature
convergence.

This diagnosis supports the design philosophy of D$^3$PO. Rather than seeking
expressiveness through nonlinear transformations, D$^3$PO preserves the exact
structure of PPO's clipped surrogate for each objective independently. The clipping
operator retains its intended semantics — bounding per-objective updates to a
calibrated trust region — and the only modification is the order of operations:
preference weights are applied after rather than before the clipping decision. The
resulting nonlinearity is not injected but emergent: it arises from the interaction
between independent clipping decisions across objectives and the preference-weighted
aggregation of already-stabilized surrogates. This emergent nonlinearity is sufficient
in practice to recover broad, high-quality Pareto fronts on all evaluated benchmarks,
without the instability that accompanies explicit nonlinear scalarization.

The broader implication for multi-objective PPO is clear: the path to recovering
non-convex Pareto regions does not lie in modifying the surrogate loss, but in
the architecture and sampling strategy that determines which regions of the frontier
receive gradient pressure. The diversity regularizer in D$^3$PO addresses this at
the behavioral level, ensuring that distinct preferences map to distinct policies
without requiring any modification to the trust region structure. Future work that
seeks to recover genuinely concave frontier regions within a PPO framework should
similarly decouple the coverage mechanism from the optimization objective.

%% file: documentBody/L-demo.tex
\section{Demonstration with User Interface}

We have developed a user interface to demonstrate the behaviour of D3PO agents. There are 3 columns in the user interface. The first column shows the live policy rollout rendering. The second column shows the a line plot reward collected in every channel over time and a bar plot of the instantaneous reward at the current time step. The third column shows a slider for the objectives that are part of the environment. These sliders can change the weight value for the particular objective during the rollout to change the policy behaviour. The attached videos show demonstrations with the Mo-hopper-3D and MO-ant-3d environments. The flask file that serves this demo is part of the code and will be made public.